%% file: fixed-signSGD19.tex
\documentclass[font=11pt]{article}
\usepackage[pagebackref=true]{hyperref}
\usepackage[verbose=true,letterpaper]{geometry}
\newgeometry{
    textheight=9in,
    textwidth=5.7in,
    top=1in,
    headheight=12pt,
    headsep=25pt,
    footskip=30pt
  }
  \usepackage[dvipsnames]{xcolor}
\usepackage[english]{babel}
\usepackage{blindtext}
\definecolor{mydarkblue}{rgb}{0,0.08,0.45}
\hypersetup{ %
    colorlinks=true,
    linkcolor=mydarkblue,
    citecolor=mydarkblue,
    filecolor=mydarkblue,
    urlcolor=mydarkblue,
    pdfview=FitH}

\usepackage{natbib}
\renewcommand{\cite}[1]{\citep{#1}} 
\date{}

\input{macros.tex}

  \providecommand{\signsgd}{\textsc{signSGD}}
  \providecommand{\adam}{ADAM} 
  \providecommand{\signum}{\textsc{signSGDm}}
  \providecommand{\esignsgd}{\textsc{ef-signSGD}}
  \providecommand{\ecsgd}{\textsc{ef-SGD}}
  \providecommand{\sgdm}{\textsc{SGDm}}

\title{Error Feedback Fixes SignSGD and other Gradient Compression Schemes}

\author{
  Sai Praneeth Karimireddy, Quentin Rebjock, Sebastian U. Stich, Martin Jaggi\\\vspace*{-0.3cm}
  EPFL\vspace*{2mm}\\
  	\texttt{\small \{sai.karimireddy,~quentin.rebjock,~sebastian.stich,~martin.jaggi\}@epfl.ch} \vspace*{-2mm}
}
\begin{document}

\maketitle
\vspace*{-0.5cm}

\input{main.tex}
\bibliographystyle{plainnat}
\bibliography{papers}
\appendix
\part*{Appendix}
\input{appendix.tex}
\end{document}

%% file: macros.tex
\usepackage{amsfonts}
\usepackage{amsmath}
\usepackage{amsthm}
\usepackage{amssymb}
\usepackage{mathtools}
\usepackage{dsfont}
\usepackage[T1]{fontenc}
\usepackage{color}
\usepackage{graphicx}
\usepackage{wrapfig}
\usepackage{xparse}
\usepackage{enumitem}
\usepackage{multirow}
\usepackage{array}
\usepackage{blkarray}
\graphicspath{{figs/}}  

\usepackage{algorithm}
\usepackage{algpseudocode}
\usepackage{booktabs}
\usepackage{verbatim}
\usepackage{subcaption}

\DeclarePairedDelimiterX{\lin}[2]{\langle}{\rangle}{#1, #2}
\DeclarePairedDelimiterX{\abs}[1]{\lvert}{\rvert}{#1}
\DeclarePairedDelimiterX{\norm}[1]{\lVert}{\rVert}{#1}
\DeclarePairedDelimiterX{\cbr}[1]{\{}{\}}{#1} 
\DeclarePairedDelimiterX{\rbr}[1]{(}{)}{#1} 
\DeclarePairedDelimiterX{\sbr}[1]{[}{]}{#1} 

  \providecommand{\R}{\mathbb{R}} 
  \providecommand{\real}{\mathbb{R}} 

  \DeclareMathOperator{\expect}{\mathbb{E}}

  \DeclareMathOperator{\sgn}{sign}
  \makeatletter
  \def\sign{\@ifnextchar*{\@sgnargscaled}{\@ifnextchar[{\sgnargscaleas}{\@ifnextchar{\bgroup}{\@sgnarg}{\sgn} }}}
  \def\@sgnarg#1{\sgn\rbr{#1}}
  \def\@sgnargscaled#1{\sgn\rbr*{#1}}
  \def\@sgnargscaleas[#1]#2{\sgn\rbr[#1]{#2}}
  \makeatother

  \DeclareMathOperator*{\argmin}{arg\,min}

  \providecommand{\0}{\mathbf{0}}
  
  \renewcommand{\aa}{\mathbf{a}}
  \providecommand{\bb}{\mathbf{b}}

  \providecommand{\ee}{\mathbf{e}}

  \renewcommand{\gg}{\mathbf{g}}

  \providecommand{\mm}{\mathbf{m}}

  \providecommand{\pp}{\mathbf{p}}

  \renewcommand{\ss}{\mathbf{s}}

  \providecommand{\vv}{\mathbf{v}}
  
  \providecommand{\xx}{\mathbf{x}}
  \providecommand{\yy}{\mathbf{y}}
  
  \providecommand{\txx}{\tilde\xx}
  \providecommand{\balpha}{\boldsymbol{\alpha}}

  \providecommand{\mA}{\mathbf{A}}

  \providecommand{\mG}{\mathbf{G}}

  \providecommand{\mX}{\mathbf{X}}

  \providecommand{\cC}{\mathcal{C}}

  \providecommand{\cO}{\mathcal{O}}
  
  \providecommand{\cQ}{\mathcal{Q}}

  \providecommand{\cU}{\mathcal{U}}

\newtheorem{theorem}{Theorem}

\newtheorem{lemma}{Lemma}
\newtheorem{remark}[lemma]{Remark}

\newtheorem{definition}{Definition}

\newtheorem{assumption}[definition]{Assumption}


%% file: main.tex
\begin{abstract}
Sign-based algorithms (e.g. {\signsgd}) have been proposed as a biased gradient compression technique to alleviate the communication bottleneck in training large neural networks across multiple workers. We show simple convex counter-examples where signSGD does not converge to the optimum.
Further, even when it does converge, signSGD may generalize poorly when compared with SGD.  These issues arise because of the biased nature of the sign compression operator.

We then show that using error-feedback, i.e. incorporating the error made by the compression operator into the next step, overcomes these issues. We prove that our algorithm (\ecsgd) with arbitrary compression operator achieves the \emph{same rate of convergence} as SGD without any additional assumptions. Thus \ecsgd\ achieves gradient compression \emph{for free}. Our experiments thoroughly substantiate the theory and show that error-feedback improves both convergence and generalization. Code can be found at \url{https://github.com/epfml/error-feedback-SGD}.
\end{abstract}


\section{Introduction}
Stochastic optimization algorithms \cite{Bottou2010:sgd} which are amenable to large-scale parallelization, taking advantage of massive computational resources \cite{krizhevsky2012imagenet,dean2012large} have been at the core of significant recent progress in deep learning \cite{schmidhuber2015deep,lecun2015deep}.
One such example is the {\signsgd} algorithm and its variants, c.f.~\cite{seide20141,bernstein2018signsgd,bernstein2019iclr}.

To minimize a continuous (possibly) non-convex function $f \colon \real^d \to \real$, the classic stochastic gradient algorithm (SGD) \cite{robbins1951stochastic} performs iterations of the form
\begin{equation}\label{eqn:sgd}\tag{SGD}
  \xx_{t+1} := \xx_t - \gamma\, \gg_t\,,
\end{equation}
where $\gamma \in \real$ is the step-size (or learning-rate) and $\gg_t$ is the stochastic gradient such that $\expect\sbr{\gg_t} = \nabla f(\xx_t)$.

Methods performing updates only based on the \emph{sign} of each coordinate of the gradient have recently gaining popularity for training deep learning models \cite{seide20141,Carlson:2015to,wen2017terngrad,balles2017dissecting,bernstein2018signsgd,zaheer2018adaptive,liu2018signsgd}.
For example, the update step of {\signsgd} is given by:
\begin{equation}\label{eqn:signsgd}\tag{\signsgd}
  \xx_{t+1} := \xx_t - \gamma \sign{\gg_t}\,.\hspace{-8mm}
\end{equation}
Such sign-based algorithms are particularly interesting since they can be viewed through two lenses: as i) approximations of adaptive gradient methods such as {\adam} \cite{balles2017dissecting}, and also a ii) communication efficient gradient compression scheme \cite{seide20141}.
However, we show that a severe handicap of sign-based algorithms is that they \emph{do not} converge in general. To substantiate this claim, we present in this work simple convex counter-examples where {\signsgd} cannot converge. The main reasons being that the $\sign$ operator loses information about, i.e. `forgets', i) the magnitude, as well as ii) the direction of $\gg_t$.
We present an elegant solution that provably fixes these problems of {\signsgd}, namely algorithms with \emph{error-feedback}.

\paragraph{Error-feedback.}
We demonstrate that the aforementioned problems of {\signsgd} can be fixed by i) scaling the signed vector by the norm of the gradient to ensure the magnitude of the gradient is not forgotten, and ii) locally storing the difference between the actual and compressed gradient, and iii) adding it back into the next step so that the correct direction is not forgotten. We call our `fixed' method {\esignsgd} (Algorithm~\ref{alg:esignsgd}).
%

\begin{algorithm}[t]
	\caption{{\esignsgd} ({\signsgd} with Error-Feedb.)}
		\label{alg:esignsgd}
	\begin{algorithmic}[1]
	\State \textbf{Input:} learning rate $\gamma$, initial iterate $\xx_0 \in \R^d$, $\ee_0 = \0$
	\For{$t = 0,\dots, T-1$}
    \State $\gg_t := \text{stochasticGradient}(\xx_t)$
    \State $\pp_{t} := \gamma \gg_t + \ee_t$ \hfill $\triangleright$ {\it error correction}
    \State $\Delta_t := (\norm{\pp_t}_1/d) \sign{\pp_t}$ \hfill $\triangleright$ {\it compression}
    \State $\xx_{t+1} := \xx_t - \Delta_t$ \hfill $\triangleright$ {\it update iterate}
    \State $\ee_{t+1} := \pp_t - \Delta_t$ \hfill $\triangleright$ {\it update residual error}
    \EndFor
    \end{algorithmic}
\end{algorithm}

In Algorithm~\ref{alg:esignsgd}, $\ee_t$ denotes the accumulated error from all quantization/compression steps. This residual error is added to the gradient step $\gamma \gg_t$ to obtain the corrected direction~$\pp_t$. When compressing $\pp_t$, the signed vector is again scaled by $\norm{\pp_t}_1$ and hence does not lose information about the magnitude. Note that our algorithm does not introduce any additional parameters and requires only the step-size $\gamma$.

\paragraph{Our contributions.}
We show that naively using biased gradient compression schemes (such as e.g. {\signsgd}) can lead to algorithms which may not generalize or even converge. We show both theoretically and experimentally that simply adding error-feedback solves such problems and recovers the performance of full SGD, thereby saving on communication costs. We state our results for {\signsgd} to ease our exposition but our positive results are valid for general compression schemes, and our counterexamples extend to {\signsgd} with momentum, multiple worker settings, and even other biased compression schemes. More specifically our contributions are:
\vspace{-1mm}
\begin{enumerate}[itemsep=3pt]
  \item We construct a simple convex non-smooth counterexample where {\signsgd} cannot converge, even with the full batch (sub)-gradient and tuning the step-size. Another counterexample for a wide class of smooth convex functions proves that {\signsgd} with stochastic gradients cannot converge with batch-size one.

  \item We prove that by incorporating error-feedback, {\signsgd}---as well as any other gradient compression schemes---always converge. Further, our theoretical analysis for non-convex smooth functions recovers the same rate as SGD, i.e. we get gradient compression \emph{for free}.

  \item We show that our algorithm {\esignsgd} which incorporates error-feedback approaches the linear span of the past gradients. Therefore, unlike {\signsgd}, {\esignsgd} converges to the max-margin solution in over-parameterized least-squares. This provides a theoretical justification for why {\esignsgd} can be expected to generalize much better than {\signsgd}.

  \item We show extensive experiments on CIFAR10 and CIFAR100 using Resnet and VGG architectures demonstrating that {\esignsgd} indeed significantly outperforms {\signsgd}, and matches SGD both on test as well as train datasets while reducing communication by a factor of $\sim 64\times$.
\end{enumerate}


\section{Significance and Related Work}
\paragraph{Relation to adaptive methods.} Introduced in \cite{kingma2014adam}, {\adam} has gained immense popularity as the algorithm of choice for adaptive stochastic optimization for its perceived lack of need for parameter-tuning. However since, the convergence \cite{reddi2018convergence} as well the generalization performance \cite{wilson2017marginal} of such adaptive algorithms has been called into question. Understanding when {\adam} performs poorly and providing a principled `fix' for these cases is crucial given its importance as the algorithm of choice for many researchers. It was recently noted by \citet{balles2017dissecting} that the behavior of {\adam} is in fact identical to that of {\signsgd} with \emph{momentum}: More formally, the {\signum} algorithm (referred to as 'signum' by \citet{bernstein2018signsgd,bernstein2019iclr}) adds momentum to the {\signsgd} update as:
\begin{equation}\label{eqn:signum}\tag{\signum}
\begin{split}
  \mm_{t+1} &:= \gg_t + \beta\mm_t\\
  \xx_{t+1} &:= \xx_t - \gamma \sign{\mm_{t+1}}\,.\hspace{-8mm}
\end{split}
\end{equation}
for parameter $\beta > 0$. This connection between signed methods and fast stochastic algorithms is not surprising since sign-based gradient methods were first studied as a way to speed up SGD \cite{riedmiller1993direct}. Given their similarity, understanding the behavior of {\signsgd} and {\signum} can help shed light on the convergence of {\adam}.

\paragraph{Relation to gradient compression methods.} As the size of the models keeps getting bigger, the training process can often take days or even weeks \cite{dean2012large}. This process can be significantly accelerated by massive parallelization \cite{li2014scaling,goyal2017accurate}. However, at these scales communication of the gradients between the machines becomes a bottleneck hindering us from making full use of the impressive computational resources available in today's data centers \cite{chilimbi2014project,seide20141,strom2015scalable}. A simple solution to alleviate this bottleneck is to compress the gradient and reduce the number of bits transmitted. While the analyses of such methods have largely been restricted to \emph{unbiased} compression schemes \cite{alistarh2017quantized,wen2017terngrad,wang2018atomo}, biased schemes which perform extreme compression practically perform much better \cite{seide20141,strom2015scalable,lin2017deep}---often \emph{without any loss} in convergence or accuracy. Of these, \cite{seide20141,strom2015scalable,wen2017terngrad} are all sign-based compression schemes. Interestingly, all the practical works \cite{seide20141,strom2015scalable,lin2017deep} use some form of error-feedback.

\paragraph{Error-feedback.} The idea of error-feedback was, as far as we are aware, first introduced in 1-bit SGD \cite{seide20141,strom2015scalable}. The algorithm 1-bit SGD is very similar to our {\esignsgd} algorithm, but tailored for the specific recurrent network studied there. \cite{wu2018error} analyze unbiased compressors with a form of error-feedback involving two additional hyper-parameters and restricted to quadtratic functions. Though not presented as such, the `momentum correction' used in \cite{lin2017deep} is a variant of error-feedback. However the error-feedback is not on the vanilla SGD algorithm, but on SGD {with momentum}. Recently, \citet{stich2018sparsified} conducted the first theoretical analysis of error-feedback for compressed gradient methods (they call it `memory') in the strongly convex case. Our convergence results can be seen as an extension of theirs to the non-convex and weakly convex cases.

\paragraph{Generalization of deep learning methods.}
Deep networks are almost always over-parameterized and are known to be able to fit arbitrary data and always achieve zero training error \cite{zhang2016understanding}. This ability of deep networks to generalize well on real data, while simultaneously being able to fit arbitrary data has recently received a lot of attention (e.g. \citet{soudry2018implicit,dinh2017sharp,zhang2018study,arpit2017closer,kawaguchi2017generalization}).
{\signsgd} and {\adam} are empirically known to generalize worse than SGD \cite{wilson2017marginal,balles2017dissecting}. A number of recent papers try close this gap for {\adam}. \citet{luo2019adaptive} show that by bounding the adaptive step-sizes in {\adam} leads to closing the generalization gap. They require new hyper-parameters on top of {\adam} to adaptively tune these bounds on the step-sizes. \citet{chen2019padam} interpolate between SGD and {\adam} using a new hyper-parameter $p$ and show that tuning this can recover performance of SGD. \citet{zaheer2018adaptive} introduce a new adaptive algorithm which is closer to Adagrad \cite{duchi2011adaptive}. Similarly, well-tuned {\adam} (where all the hyper-parameters and not just the learning rate are tuned) is also known to close the generalization gap \cite{Gugger2018AdamW}. In all of these algorithms, new hyper-parameters are introduced which essentially control the effect of the adaptivity. Thus they require additional tuning while the improvement upon traditional SGD is questionable. We are not aware of other work bridging the generalization gap in sign-based methods.


\section{Counterexamples for SignSGD}
In this section we study the limitations of {\signsgd}.
Under benign conditions---for example if i) the function $f$ is smooth, and ii) the stochastic noise is gaussian or an extemely large batch-size is used (equal to the total number of iterations)---the algorithm can be shown to converge \cite{bernstein2018signsgd,bernstein2019iclr}. However, we show that {\signsgd} does not converge under more standard assumptions. We demonstrate this first on a few pedagogic examples and later also for realistic and general sum-structured loss functions.
If we use a fixed step-size $\gamma \geq 0$%
, {\signsgd} does not converge even for simple one-dimensional linear functions.
\paragraph{Counterexample 1.}
For $x \in \R$ consider the constrained problem\vspace{-1mm}
\[
  \min_{x \in [-1,1]} \left[ f(x) := \tfrac{1}{4}x \right] \,,
  \]
with minimum at $x^\star =-1$. Assume stochastic gradients are given as (note that $f(x) = \tfrac{1}{4}(4x - x - x - x)$)
\begin{align*}
  g = \begin{cases}
  4, &\text{ with prob. } \tfrac{1}{4}\\
  -1, &\text{ with prob. } \tfrac{3}{4}
  \end{cases} & &\text{with} & & \expect[g] = \nabla f(x)\,.
\end{align*}
For SGD with any step-size $\gamma$,\vspace{-2mm}
\[
  \expect_t[f(x_{t+1})] = \frac{1}{4}\rbr*{x_t - \gamma \expect[g]} = f(x_t) - \frac{\gamma}{16}\,.
\]
On the other hand, for {\signsgd} with any fixed $\gamma$,
\[
  \expect_t[f(x_{t+1})] = \frac{1}{4}\rbr*{x_t - \gamma \expect[\sign{g}]} = f(x_t) + \frac{\gamma}{8}\,,
\]
i.e. the objective function increases in expectation for $\gamma \geq 0$.
\begin{remark}
  In the above example, we exploit  that the $\sign$ operator loses track of the magnitude of the stochastic gradient. Also note that our noise is bimodal. The counter-examples for the convergence of {\adam} \cite{reddi2018convergence,luo2019adaptive} also use similar ideas. Such examples were previouslyt noted for {\signsgd} by \cite{bernstein2019iclr}.
\end{remark}

In the example above the step-size $\gamma$ was fixed. However increasing batch-size or tuning the step-size may still allow convergence. Next we show that even with adaptive step-sizes (e.g. decreasing, or adaptively chosen optimal step-sizes) {\signsgd} does not converge. This even holds if the full (sub)-gradient is available (non-stochastic case).

\paragraph{Counterexample 2.}
\begin{figure}[t]
  \begin{center}\hspace{-13mm}
    \def\svgwidth{0.35\columnwidth}
    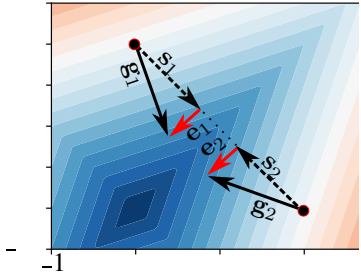
  \end{center}\vspace*{-0.5cm}
  \caption{The gradients $\gg$ (in solid black), signed gradient direction $\ss = \protect\sign{\gg}$ (in dashed black), and the error $\ee$ (in red) are plotted for $\epsilon = 0.5$. {\signsgd} moves only along $\ss = \pm (1,-1)$ while the error $\ee$ is ignored.}
  \label{fig:counter-cases}\vspace*{-4mm}
\end{figure}
For $\xx \in \R^2$ consider the following non-smooth convex problem with $\xx^\star = (0, 0)^\top$:\vspace{-1mm}
\[
 \min_{\xx \in \R^2}  \Big[ f(\xx) := \epsilon \abs{x_1 + x_2} + \abs{x_1 - x_2} \Big] \,,
\]
for parameter $0 < \epsilon < 1$ and subgradient
\[
\gg(\xx) = \sign(x_1 + x_2) \cdot \epsilon \begin{pmatrix} 1 \\ 1 \end{pmatrix} + \sign(x_1-x_2) \begin{pmatrix} 1 \\ -1 \end{pmatrix}\,.
\]
See Fig. \ref{fig:counter-cases}. The iterates of {\signsgd} started at $\xx_0 = (1,1)^\top$ lie along the line $x_1 + x_2 = 2$.
Note that for any $\xx$ s.t. $x_1 + x_2 >0$,  $\sign{\gg(\xx)} \! = \! \pm (1,-1)^\top$, and hence $x_1 + x_2$ remains constant among the iterations of {\signsgd}.
Consequently, for any step-size sequence $\gamma_t$, $f(\xx_t) \geq f(\xx_0)$.
\begin{remark}
In this example, we exploit the fact that the sign operator is a biased approximation of the gradient---it consistently ignores the direction $\ee = \epsilon(1,1)^\top$ (see Fig \ref{fig:counter-cases}). Tuning the step-size would not help either.
\end{remark}

One might wonder if the smooth-case is easier. Unfortunately, the previous example can easily be extended to show that {\signsgd} with stochastic gradients may not converge even for smooth functions.
\paragraph{Counterexample 3.}
For $\xx \in \R^2$ consider the following least-squares problem with
$\xx^\star = (0,0)^\top$:
\[
  \min_{\xx \in \R^2} \left[ f(\xx) := (\lin{\aa_1}{\xx})^2 + (\lin{\aa_2}{\xx})^2\right]\,, \text{ where}\vspace{-1mm}
\]
\[ \aa_{1,2} := \pm (1, -1) + \epsilon(1,1)\,,
\]
for parameter $0 < \epsilon < 1$ and stochastic gradient $\gg(\xx)= \nabla_{\xx} (\lin{\aa_1}{\xx})^2$ with prob.\ $\tfrac{1}{2}$ and $\gg(\xx) = \nabla_{\xx} (\lin{\aa_2}{\xx})^2$ with prob.\ $\tfrac{1}{2}$. The stochastic gradient is then either $e \aa_1$ or $e \aa_2$ for some scalar $e$.
Exactly as in the non-smooth case, for $\xx_0 = (1,1)^\top$, the sign of the gradient $\sign{\gg} = \pm (1, -1)$. Hence {\signsgd} with any step-size sequence remains stuck along the line $x_1 + x_2 = 2$ and $f(\xx_t) \geq f(\xx_0)$ a.s.

We can generalize the above counter-example to arbitrary dimensions and loss functions.
\begin{theorem}\label{thm:counter}
  Suppose that scalar loss functions $\{l_i \colon \R \to \R \}_{i=1}^n$ and data-points $\{\aa_i\}_{i=1}^n \in \real^d$ for $d \geq 2$ satisfy: i) $f(\xx) := \sum_{i=1}^n l_i(\lin{\aa_i}{\xx})$ has a unique optimum at $\xx^\star$, and ii) there exists $\ss \in \{-1,1\}^d$ such that $\sign{\aa_i} = \pm \ss$ for all $i$. Then {\signsgd} with batch-size 1 and stochastic gradients $g(\xx) = \nabla_{\xx} l_i(\lin{\aa_i}{\xx})$ for $i$ chosen uniformly at random does not converge to $\xx^\star$ a.s. for any adaptive sequence of step-sizes, even with random initialization.
\end{theorem}
%


\section{Convergence of Compressed Methods}
We show the rather surprising result that incorporating error-feedback is sufficient to ensure that the algorithm converges at a rate which matches that of SGD. In this section we consider a general gradient compression scheme.

\vspace{5mm} 

\begin{algorithm}
	\caption{{\ecsgd} (Compr. SGD with Error-Feedback)\!\!\!\!}
		\label{alg:ecsgd}
	\begin{algorithmic}[1]
		\State \textbf{Input:} learning rate $\gamma$, compressor $\cC(\cdot)$,  $\xx_0 \in \R^d$
		\State \textbf{Initialize:} $\ee_0=\0 \in \R^d$
 	\For{$t = 0,\dots, T-1$}
    \State $\gg_t := \text{stochasticGradient}(\xx_t)$
    \State $\pp_t := \gamma \gg_t + \ee_t$ \hfill $\triangleright$ {\it error correction}
    \State $\Delta_{t} := \cC(\pp_t)$ \hfill $\triangleright$ {\it compression}
    \State $\xx_{t+1} := \xx_t - \Delta_t$ \hfill $\triangleright$ {\it update iterate}
    \State $\ee_{t+1} := \pp_t - \Delta_t $  \hfill $\triangleright$ {\it update residual error}
    \EndFor
    \end{algorithmic}
\end{algorithm}

\pagebreak

We generalize the notion of a compressor from \cite{stich2018sparsified}.
\begin{assumption}[Compressor]\label{asm:compressor}
An operator $\cC: \real^d \rightarrow \real^d$ is a $\delta$-approximate compressor over $\cQ$ for $\delta \in (0,1]$ if
\[
  \norm{\cC(\xx) - \xx}^2_2 \leq (1 - \delta)\norm{\xx}^2_2,\ \forall \xx \in \cQ\,.
\]
\end{assumption}
Note that $\delta = 1$ implies that $\cC(\xx) = \xx$. Examples of compressors include: i) the $\sign$ operator, ii) top-$k$ which selects $k$ coordinates with the largest absolute value while zero-ing out the rest \cite{lin2017deep,stich2018sparsified}, iii) $k$-PCA which approximates a matrix $\mX$ with its top $k$ eigenvectors \cite{wang2018atomo}.
Randomized compressors satisfying the assumption in expectation are also allowed.

We now state a key lemma that shows that the residual errors maintained in Algorithm~\ref{alg:ecsgd} do not accumulate too much.
\begin{lemma}[Error is bounded]\label{lem:bounded-error}
  Assume that $\expect[\norm{\gg_t}^2] \leq \sigma^2$ for all $t \geq 0$. Then at any iteration $t$ of {\ecsgd}, the norm of the error $\ee_t$ in Algorithm~\ref{alg:ecsgd} is bounded:\vspace{-1mm}
    \begin{align*}
     \expect \norm{\ee_t}^2_2 \leq \frac{4(1 - \delta) \gamma^2 \sigma^2}{\delta^2}\,, & &\forall t \geq 0\,.
    \end{align*}
\end{lemma}
If $\delta = 1$, then $\norm{\ee_t} = 0$ and the error is zero as expected.

We employ standard assumptions of smoothness of the loss function and the variance of the stochastic gradient.
\begin{assumption}[Smoothness]\label{asm:smoothness}
  A function $f \colon \R^d \to \R$ is $L$\nobreakdash-smooth if for all $\xx$, $\yy \in \R^d$ the following holds:\vspace{-2mm}
  \[
    \abs*{f(\yy) - \rbr*{f(\xx) + \lin*{\nabla f(\xx)}{\yy - \xx}}} \leq \frac{L}{2}\norm*{\yy - \xx}^2_2\,.
  \]
\end{assumption} %
\begin{assumption}[Moment bound]\label{asm:moment}
  For any $\xx$, our query for a stochastic gradient returns $\gg$ such that\vspace{-1mm}
  \[
    \expect[\gg] = \nabla f(\xx) \quad \text{and} \quad \expect \norm*{\gg}^2_2  \leq \sigma^2\,.\vspace{-1mm}
  \]
\end{assumption}
Given these assumptions, we can formally state our theorem followed by a sketch of the proof.
\begin{theorem}[Non-convex convergence of {\ecsgd}]\label{thm:ecsgd-non-convex}
Let $\{\xx_t\}_{t \geq 0}$ denote the iterates of Algorithm~\ref{alg:ecsgd} for any step-size $\gamma>0$. Under Assumptions~\ref{asm:compressor},~\ref{asm:smoothness},~and~\ref{asm:moment},
  \[
    \min_{t \in [T]}\expect\sbr{\norm{\nabla f(\xx_t)}^2} \leq \frac{2 f_0}{\gamma (T+1)} + \frac{\gamma L \sigma^2}{2}  + \frac{4 \gamma^2 L^2 \sigma^2(1 - \delta)}{\delta^2}  \,,
  \]
  with $f_0 := f(\xx_0) - f^\star$.
\end{theorem}
\begin{proof}[Proof Sketch]
  Intuitively, the condition that $\cC(\cdot)$ is a $\delta$-approximate compressor implies that at each iteration a $\delta$-fraction of the gradient information is sent. The rest is added to $\ee_t$ to be transmitted later. Eventually, all the gradient information is transmitted---albeit with a delay which depends on~$\delta$. Thus {\ecsgd} can intuitively be viewed as a delayed gradient method. If the function is smooth, the gradient does not change quickly and so the delay does not significantly matter.

  More formally, consider the error-corrected sequence $\txx_t$ which represents~$\xx_t$ with the `delayed' information added:\vspace{-1mm}
  \[
  \txx_t := \xx_t - \ee_t\,.\vspace{-1mm}
  \]
  It satisfies the recurrence\vspace{-1mm}
  \[
    \txx_{t+1} = \xx_{t} - \ee_{t+1} - \cC(\pp_t) = \xx_t - \pp_t = \txx_t - \gamma \gg_t\,.
  \]
  If $\xx_t$ was exactly equal to~$\txx_t$ (i.e. there was zero `delay'), then we could proceed with the standard proof of SGD. We instead rely on Lemma \ref{lem:bounded-error} which shows $\txx_t \approx \xx_t$ and on the smoothness of $f$ which shows $\nabla f(\xx_t) \approx \nabla f(\txx_t)$.
\end{proof}

\begin{remark}\label{rem:ecsgd-non-convex-rate}
If we substitute $\gamma := \frac{1}{\sqrt{T+1}}$ in Theorem \ref{thm:ecsgd-non-convex}, we get\vspace{-1mm}
\[
  \min_{t \in [T]}\expect\sbr{\norm{\nabla f(\xx_t)}^2} \leq \frac{4(f(\xx_0) - f^\star)  + L \sigma^2}{2\sqrt{T+1}}  + \frac{4 L^2 \sigma^2(1 - \delta)}{\delta^2 (T+1)}
\]
In the above rate, the compression factor $\delta$ only appears in the higher order $\mathcal{O}(1/T)$ term. For comparison, SGD under the exact same assumptions achieves
\[
  \min_{t \in [T]}\expect\sbr{\norm{\nabla f(\xx_t)}^2} \leq \frac{2(f(\xx_0) - f^\star)  + L \sigma^2}{2\sqrt{T+1}} \,.
\]
  This means that after $T \geq \cO(1/\delta^2)$ iterations the second term becomes negligible and the rate of convergence catches up with full SGD---this is usually true after just the first few epochs. Thus we prove that compressing the gradient does not change the asymptotic rate of SGD.\footnote{The astute reader would have observed that the asymptotic rate for {\ecsgd} is in fact 2 times slower than SGD. This is just for simplicity of presentation and can easily be fixed with a tighter analysis.}
\end{remark}
\begin{remark}\label{rem:unbiased-estimator}
  The use of error-feedback was motivated by our counter-examples for biased compression schemes. However our rates show that even if using an unbiased compression (e.g. QSGD \cite{alistarh2017quantized}), using error-feedback gives significantly better rates. Suppose we are given an unbiased compressor $cU(\cdot)$ such that
  $\expect[\cU(\xx)] = \xx$ and $\expect\sbr*{\norm{\cU(\xx)}_2^2} \leq k \norm{\xx}^2$. Then without feedback, using standard analysis (e.g. \cite{alistarh2017quantized}) the algorithm converges $k$ times slower:
  \[
    \min_{t \in [T]}\expect\sbr{\norm{\nabla f(\xx_t)}^2} \leq \frac{(f(\xx_0) - f^\star)  + L k \sigma^2}{2\sqrt{T+1}} \,.
  \]
  Instead, if we use $\cC(\xx) = \frac{1}{k}\cU(\xx)$ with error-feedback, we would achieve\vspace{-2mm}
  \[
    \min_{t \in [T]}\expect\sbr{\norm{\nabla f(\xx_t)}^2} \leq \frac{4(f(\xx_0) - f^\star)  + L\sigma^2}{2\sqrt{T+1}} + \frac{2 L^2 \sigma^2 k^2}{T+1} \,,
  \]
  thereby pushing the dependence on $k$ into the higher order $\mathcal{O}(1/T)$ term.
\end{remark}

Our counter-examples showed that biased compressors may not converge for non-smooth functions. Below we prove that adding error-feedback ensures convergence under standard assumptions even for non-smooth functions.
\begin{theorem}[Non-smooth convergence of {\ecsgd}]\label{thm:ecsgd-convex}
Let $\{\xx_t\}_{t \geq 0}$ denote the iterates of Algorithm~\ref{alg:ecsgd} for any step-size $\gamma >0$ and define $\bar\xx_t = \frac{1}{T}\sum_{t=0}^{T}\xx_t$. Given that $f$ is convex and Assumptions~\ref{asm:compressor}~and~\ref{asm:moment} hold,\vspace{-1mm}
  \[
    \expect[f(\bar\xx_t)] - f^\star \leq \frac{\norm{\xx_0 - \xx^\star}^2}{2\gamma (T+1)} + \gamma \sigma^2\rbr*{\frac{1}{2} + \frac{2\sqrt{1 - \delta}}{\delta}} \,.
  \]
\end{theorem}

\begin{remark}\label{rem:ecgsd-non-smooth-rate}
  By picking the optimal $\gamma = \mathcal{O}(1/\sqrt{T})$, we see that\vspace{-3mm}
  \[
    \expect[f(\bar\xx_t)] - f^\star \leq \frac{\norm{\xx_0 - \xx^\star}\sigma}{2\sqrt{T+1}}\sqrt{1 + \frac{4\sqrt{1 - \delta}}{\delta}}\,.
  \]
  For comparison, the rate of convergence under the same assumptions for SGD is
  \[
    \expect[f(\bar\xx_t)] - f^\star \leq \frac{\norm{\xx_0 - \xx^\star}\sigma}{2\sqrt{T+1}} \,.
  \]
  For non-smooth functions, unlike in the smooth case, the compression quality $\delta$ appears directly in the leading term of the convergence rate. This is to be expected since we can no longer assume that $\nabla f(\txx_t) \approx \nabla f(\xx_t)$, which formed the crux of our argument for the smooth case.
\end{remark}

\begin{remark}
  Consider the top-1 compressor which just picks the coordinate with the largest absolute value, and zeroes out everything else. It is obvious that top-1 is a $\frac{1}{d}$-approximate compressor (cf.~\citep[Lemma A.1]{stich2018sparsified}). Running {\ecsgd} with $\cC$ as \mbox{top-1} results in a greedy coordinate algorithm. This is the first result we are aware which shows the convergence of a greedy-coordinate type algorithm on non-smooth functions.
\end{remark}

If the function is both smooth and convex, we can fall back to the analysis of \cite{stich2018sparsified} and so we won't examine it in more detail here.

\begingroup
\subsection*{Convergence of {\esignsgd}}
What do our proven rates imply for {\esignsgd} (Algorithm~\ref{alg:esignsgd}), the method of our interest here?
\begin{lemma}[Compressed sign]
  The operator $\cC(\vv) := \frac{\norm{\vv}_1}{d} \sign{\vv}$ is a\vspace{-3mm}
  \[
  \phi(\vv) = \frac{\norm{\vv}_1^2}{d \norm{\vv}_2^2}
  \]
  compressor.
\end{lemma}
\begin{figure}
\centering
    \includegraphics[width=0.34\columnwidth]{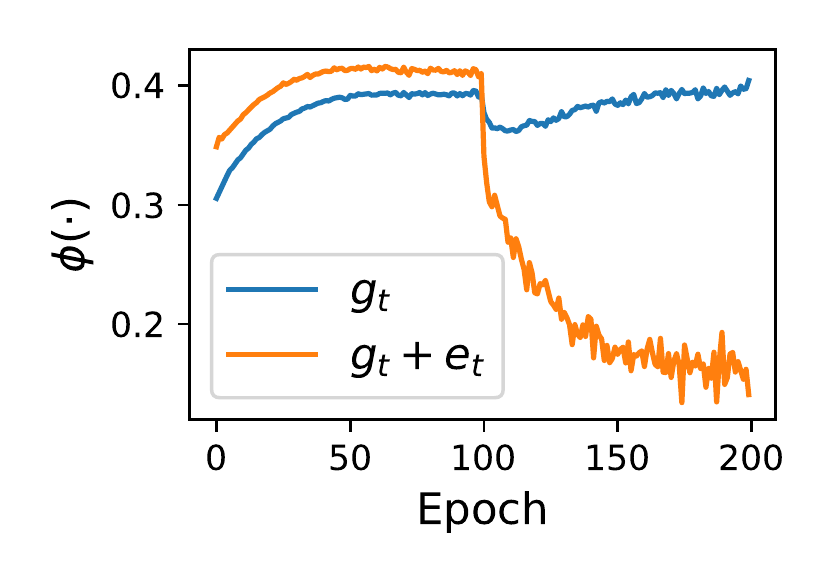}
  \caption{The density $\phi(\cdot)$ for the stochastic gradients $\gg_t$ and the error-corrected stochastic gradients $\gg_t + \ee_t$ for VGG19 on CIFAR10 and batchsize 128 (See Sec.~\ref{sec:experiments}). Minimum value of $\phi(\gg_t + \ee_t)$ is greater than $0.13$.}
  \label{fig:density}
\end{figure}
We refer to the quantity $\phi(\vv)$ as the density of $\vv$. If the vector $\vv$ had only one non-zero element, the value of $\delta$ for {\esignsgd} could be as bad as $1/d$. However, in deep learning the gradients are usually dense and hence $\phi(\vv)$ is much larger (see Fig. \ref{fig:density}). Note that for our convergence rates, it is not the density of the gradient $\gg_t$ which matters but the density of the error-corrected gradient~$\gg_t+\ee_t$.
\endgroup

%
\phantomsection \label{para:faster-than-SGD}
\paragraph{Faster convergence than SGD?}
\cite{kingma2014adam} and \cite{bernstein2018signsgd} note that different coordinates of the stochastic gradient $\gg$ may have different variances. In standard SGD, the learning rate $\gamma$ would be reduced to account for the \emph{maximum} of these coordinate-wise variances since otherwise the path might be dominated by the noise in these sparse coordinates. Instead, using coordinate-wise learning-rates like {\adam} does, or using only the coordinate-wise $\sign$ of $\gg$ as {\signsgd} does, might mitigate the effect of such `bad' coordinates by effectively scaling down the noisy coordinates. This is purported to be the reason why {\adam} and {\signsgd} can be faster than SGD on train dataset.

In {\esignsgd}, the noise from the `bad' coordinates gets accumulated in the error-term $\ee_t$ and is not forgotten or scaled down. Thus, if there are `bad' coordinates whose variance slows down convergence of SGD, {\esignsgd} should be similarly slow. Confirming this, in a toy experiment with sparse noise (Appendix \ref{subsec:sparse-noise}), SGD and {\esignsgd} converge at the same \emph{slower} rate, whereas {\signsgd} is faster. However, our real world experiments contradict this---even with the feedback, {\esignsgd} is consistently \emph{faster} than SGD, {\signsgd}, and {\signum} on training data. Thus the coordinate-wise variance adaption explanation proposed by \cite{bernstein2018signsgd,kingma2014adam} does not explain the faster convergence of {\esignsgd}, and is probably an incomplete explanation of why sign based methods or adaptive methods are faster than SGD!

\section{Generalization of SignSGD}
So far our discussion has mostly focused on the convergence of the methods i.e. their performance on training data. However for deep-learning, we actually care about their performance on test data i.e. their generalization.
It has been observed that the optimization algorithm being used significantly impacts the properties of the optima reached \cite{im2016empirical,li2018visualizing}. For instance, {\adam} and {\signsgd} are known to generalize poorly compared with SGD \cite{wilson2017marginal,balles2017dissecting}.%

The proposed explanation for this phenomenon is that in an over-parameterized setting, SGD reaches the `max-margin' solution wheras {\adam} and {\signsgd} do not \cite{zhang2016understanding,wilson2017marginal}, and \cite{balles2017dissecting}. As with the issues in convergence, the issues of {\signsgd} with generalization also turn out to be related to the biased nature of the $\sign$ operator. We explore how error-feedback may also alleviate the issues with generalization for any compression operator.
\subsection{Distance to gradient span}
Like \cite{zhang2016understanding,wilson2017marginal}, we consider an over-parameterized least-squares problem
\[
 \min_{\xx \in \R^d} \left[  f(\xx) := \norm*{\mA \xx - \yy}^2_2 \right] \,,
\]
where $\mA \in \real^{n \times d}$ for $d > n$ is the data matrix and $\yy \in \{-1,1 \}^n$ is the set of labels. The set of solutions $X^\star := \{\xx \colon f(\xx)=0\}$ of this problem forms a subspace in $\R^d$.
Of particular interest is the solution with smallest norm:\vspace{-1mm}
\[
\argmin_{\xx \colon  \in X^\star} \norm{\xx}^2 =\mA^{\dagger}\yy = \mA^\top \rbr*{\mA \mA^\top}^{-1} \yy\,,\vspace{-1mm}
\]
as this corresponds to the \emph{maximum margin} solution in the dual.

Maximizing margin is known to have a regularizing effect and is said to improve generalization \cite{valiant1984theory,cortes1995support}.

The key property that SGD (with or without momentum) trivially satisfies is that the iterates always lie in the linear span of the gradients.
\begin{lemma}\label{lem:span-pseudo-inverse}
  Given any over-parameterized least-squares problem, suppose that the iterates of the algorithm always lie in the linear span of the gradients and it converges to a 0 loss solution. This solution corresponds to the minimum norm/maximum margin solution.
\end{lemma}
If we instead use a biased compressor (e.g. {\signsgd}), it is clear that the iterate may not lie in the span of the gradients. In fact it is easy to construct examples where this happens for {\signsgd} \cite{balles2017dissecting}, as well as top-$k$ sparsification \cite{gunasekar2018characterizing}, perhaps explaining the poor generalization of these schemes.
Error-feedback is able to overcome this issue as well.
\begin{theorem}\label{thm:linear-span}
  Suppose that we run Algorithm \ref{alg:ecsgd} for $t$ iterations starting from $\xx_0 = \0$. Let $\mG_t = [g_0^\top,\dots,g_{t-1}^\top] \in \R^{d\times t}$ denote the matrix of the stochastic gradients and let $\Pi_{\mG_t} \colon \R^n \to \operatorname{Im}(\mG_t)$ denote the projection onto the range of~$\mG_t$.
  \[
    \norm*{\xx_{t} - \Pi_{\mG_t}(\xx_{t})}^2_2 \leq \norm{\ee_t}^2_2 \,.
  \]
\end{theorem}
Here $\ee_t$ is the error as defined in Algorithm \ref{alg:ecsgd}. The theorem follows directly from observing that
$
  \rbr{\xx_{t+1} + \ee_{t+1}} = \rbr{\xx_0 + \sum_{i=0}^{t}\gamma \gg_i}\,,
$
 and hence lies in the linear span of the gradients.
\begin{remark}\label{rem:generalization-gap}
  Theorem \ref{thm:linear-span} along with Lemma \ref{lem:bounded-error} implies that the iterates of Algorithm \ref{alg:ecsgd} are always close to the linear span of the gradients.\vspace{-2mm}
  \[
    \norm*{\xx_{t} - \Pi_{\mG_t}(\xx_{t})}^2_2 \leq \frac{4\gamma^2(1 - \delta)}{\delta^2}\max_{i\in[t]}\norm{\gg_i}^2\,.
  \]
  This distance further reduces as the algorithm progresses since the the step-size $\gamma$ is typically reduced.
\end{remark}
\begin{figure}[!t]
	\centering
	\includegraphics[width=0.6\linewidth]{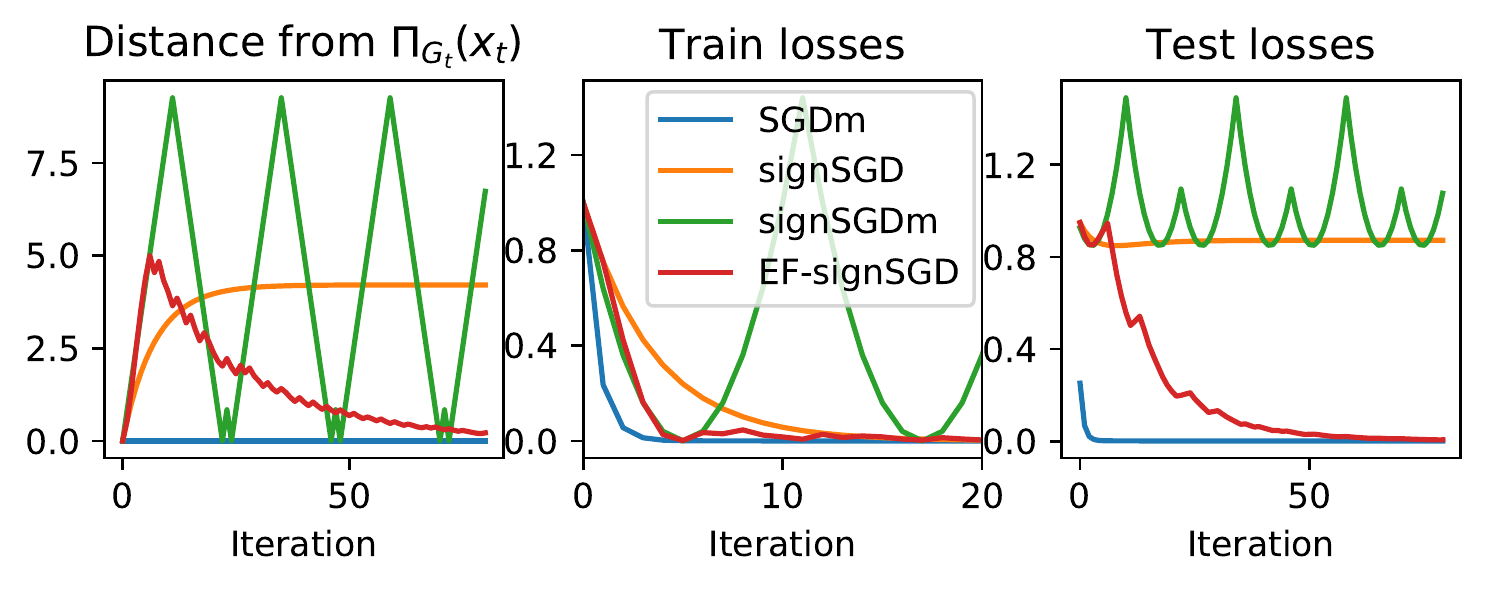}
	\caption{Left shows the distance of the iterate from the linear span of the gradients $\norm{\xx_t - \Pi_{\mG_t}(\xx_t)}$. The middle and the right plots show the train and test loss. {\signsgd} and {\signum} have a high distance to the span, and do not generalize (test loss is higher than 0.8). Distance of {\esignsgd} to the linear span (and the test loss) goes to~0.}
	\vspace{-3mm}
\label{fig:linear-span}
\end{figure}
\subsection{Simulations}\label{subesc:linear-span-simulations}
We compare the generalization of the four algorithms with full batch gradient: i) SGD ii) {\signsgd}, iii) {\signum}, and iv) {\esignsgd}. The data is generated as in \cite{wilson2017marginal} and is randomly split into test and train. The step-size $\gamma$ and (where applicable) the momentum parameter~$\beta$ are tuned to obtain fastest convergence, but the results are representative across parameter choices.

In all four cases (Fig. \ref{fig:linear-span}), the train loss quickly goes to~0. The distance to the linear span of gradients is quite large for {\signsgd} and {\signum}. For {\esignsgd}, exactly as predicted by our theory, it first increases to a certain limit and then goes to 0 as the algorithm converges. The test error, almost exactly corresponding to the distance $\norm{\xx_t - \Pi_{\mG_t}(\xx_t)}$, goes down to 0.  {\signum} oscillates significantly because of the momentum term, however the conclusion remains unchanged---the best test error is higher than 0.8. This indicates that using error-feedback might result in generalization performance comparable with SGD.

%

\section{Experiments}\label{sec:experiments}
We run experiments on deep networks to test the validity of our insights. The results confirm that i) {\esignsgd} with error feedback always outperforms the standard {\signsgd} and {\signum}, ii) the generalization gap of {\signsgd} and {\signum} vs. SGD gets larger for smaller batch sizes, iii) the performance of {\esignsgd} on the other hand is much closer to SGD, and iv) {\signsgd} behaves erratically when using small batch-sizes.

\subsection{Experimental setup}\vspace{-1mm}
All our experiments used the PyTorch framework \cite{paszke2017pytorch} on the CIFAR-10/100 dataset \cite{krizhevsky2009learning}. Each experiment is repeated three times and the results are reported in Fig \ref{fig:experiments-main}. Additional details and experiments can be found in Appendix \ref{sec:more-experiments}.

\vspace{-2mm}
\paragraph{Algorithms:} We experimentally compared the following four algorithms:
i) {\sgdm} which is SGD with momentum,
ii) (scaled){\signsgd} with step-size scaled by the $L_1$-norm of the gradient,
iii) {\signum} which is {\signsgd} using momentum, and
iv) {\esignsgd} (Alg. \ref{alg:esignsgd}). The scaled {\signsgd} performs the update\vspace{-2mm}
\[
  \xx_{t+1} := \xx_t - \gamma \frac{\norm{\gg_t}_1}{d}\sign{\gg_t}\,.\vspace{-1mm}
\]
We chose to include this in our experiments since we wanted to isolate the effects of error-feedback from that of scaling. Further we drop the unscaled {\signsgd} from our discussion here since it was observed to perform worse than the scaled version. As is standard in compression schemes \cite{seide20141,lin2017deep,wang2018atomo}, we apply our compression layer-wise. Thus the net communication for the (scaled) {\signsgd} and {\esignsgd} is $\sum_{i=1}^l(d_i + 32)$ bits where $d_i$ is the dimension of layer $i$, and $l$ is the total number of layers. If the total number of parameters is much larger than the number of layers, then the cost of the extra $32l$ bits is negligible---usually the number of parameters is three orders of magnitude more than the number of layers.

All algorithms are run for 200 epochs. The learning-rate is decimated at 100 epochs and then again at 150 epochs. The initial learning rate is tuned manually (see Appendix~\ref{sec:more-experiments}) for all algorithms using batch-size 128. For the smaller batch-sizes, the learning-rate is proportionally reduced as suggested in \cite{goyal2017accurate}. The momentum parameter~$\beta$ (where applicable) was fixed to $0.9$ and weight decay was left to the default value of $5 \times 10^{-4}$.

\vspace{-2mm}
\paragraph{Models:}  We use the VGG+BN+Dropout network on CIFAR-10 (VGG19) from \cite{simonyan2014very} and Resnet+BN+Dropout network (Resnet18) from \cite{he2016deep}. We adopt the standard data
augmentation scheme and preprocessing scheme \cite{he2016deep,he2016identity}. Our code builds upon on an open source library.\footnote{\href{https://github.com/kuangliu/pytorch-cifar}{github.com/kuangliu/pytorch-cifar}}

\begin{figure*}[!tb]
\centering
\begin{subfigure}{0.33\textwidth}
    \centering
    \includegraphics[width=1\textwidth]{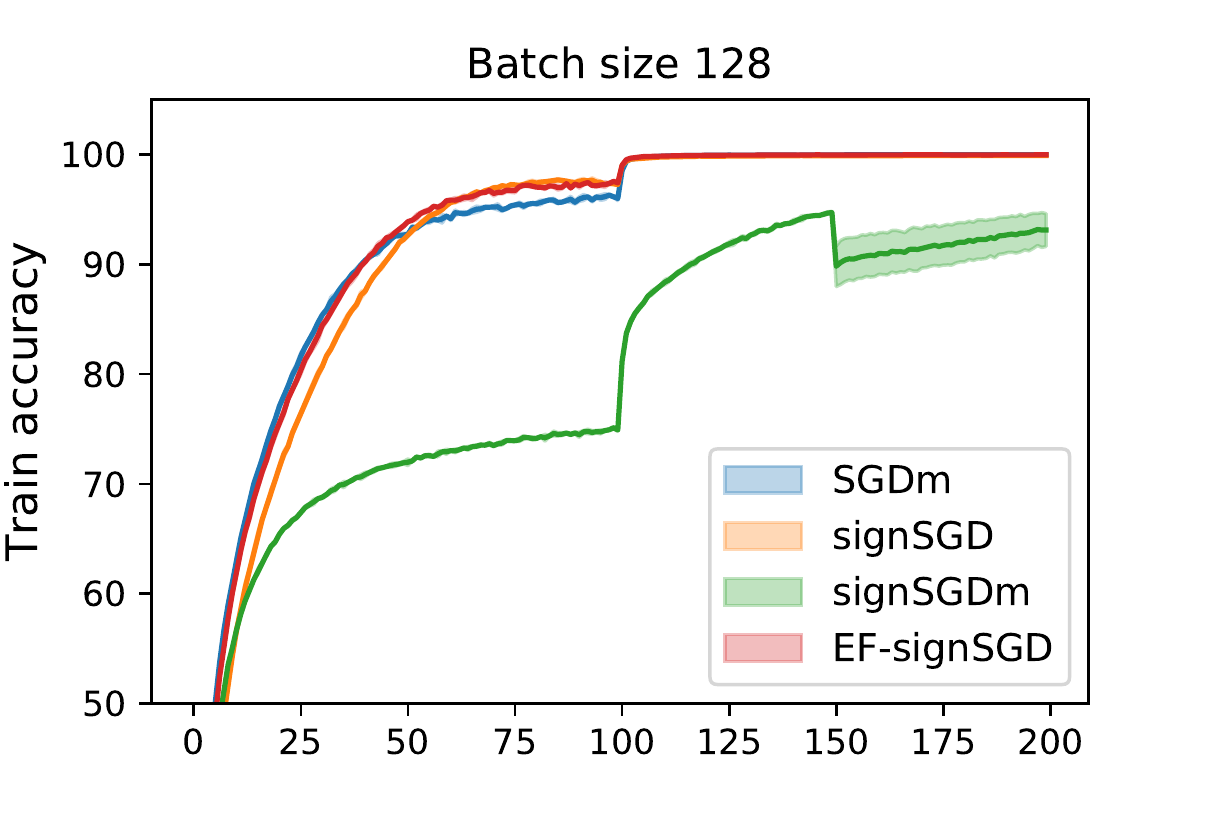}
\end{subfigure}%
\begin{subfigure}{0.33\textwidth}
    \centering
    \includegraphics[width=1\textwidth]{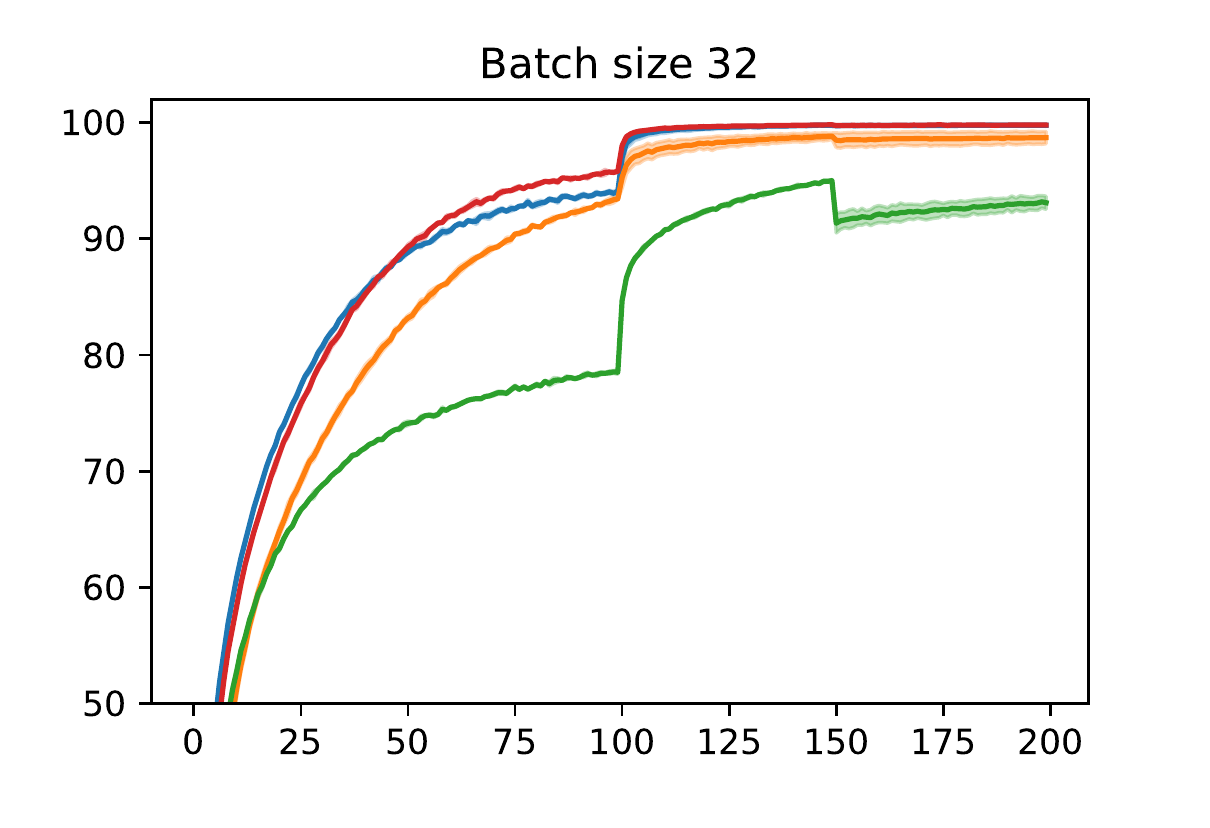}
\end{subfigure}%
\begin{subfigure}{0.33\textwidth}
    \centering
    \includegraphics[width=1\textwidth]{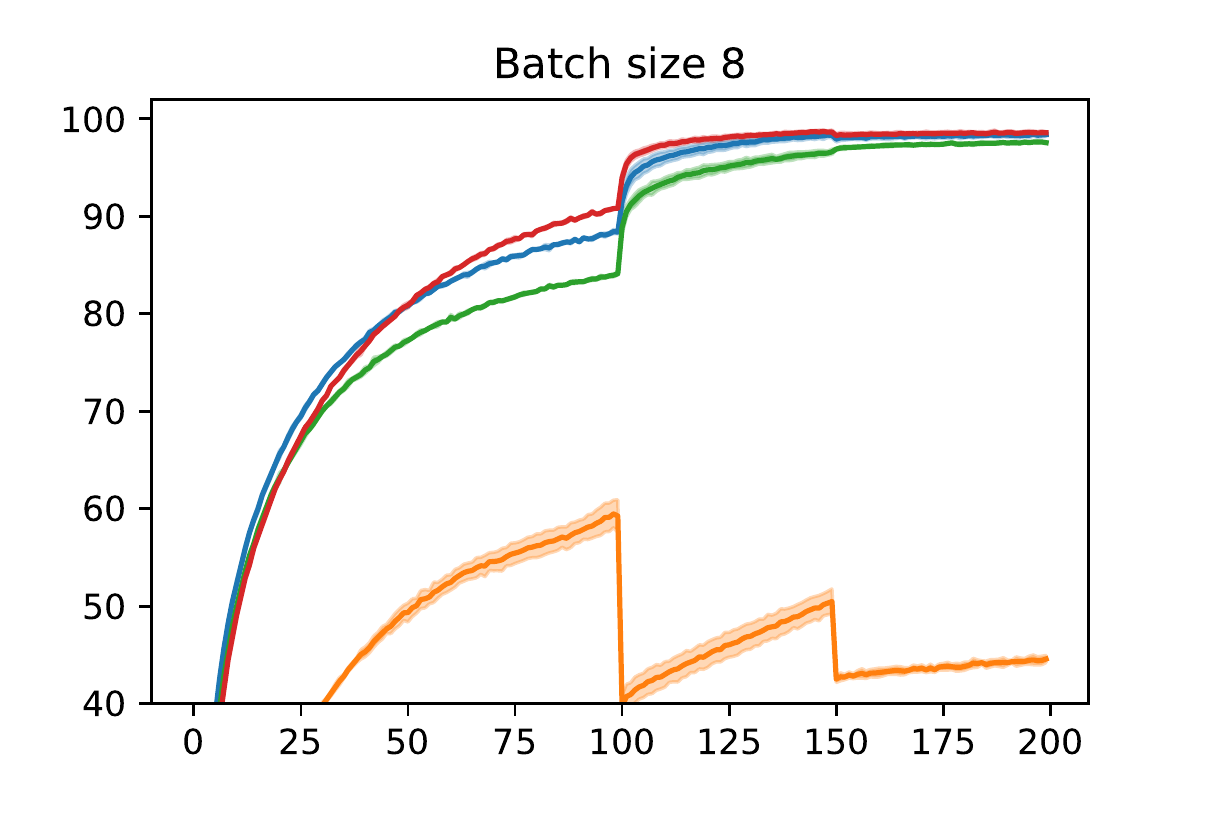}
\end{subfigure}
\begin{subfigure}{.33\textwidth}
    \centering
    \includegraphics[width=1\textwidth]{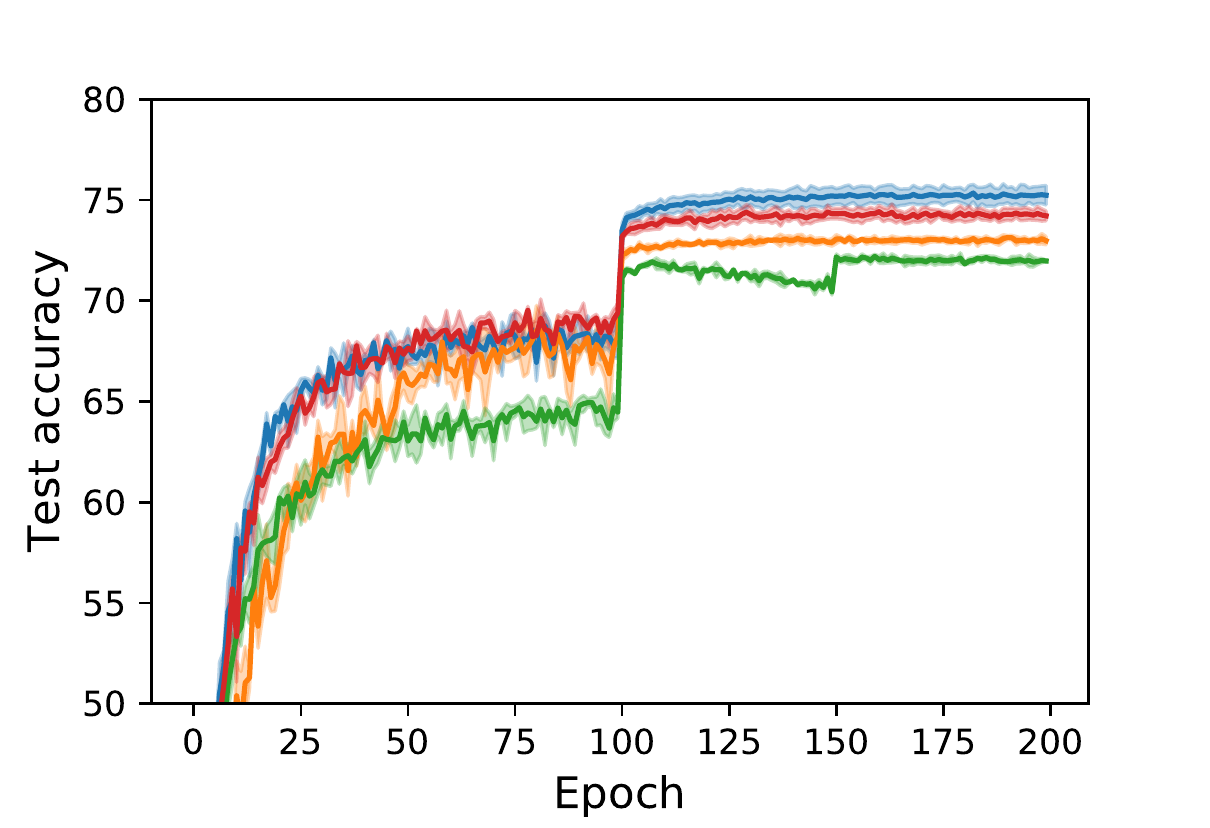}
\end{subfigure}%
\begin{subfigure}{.33\textwidth}
    \centering
    \includegraphics[width=1\textwidth]{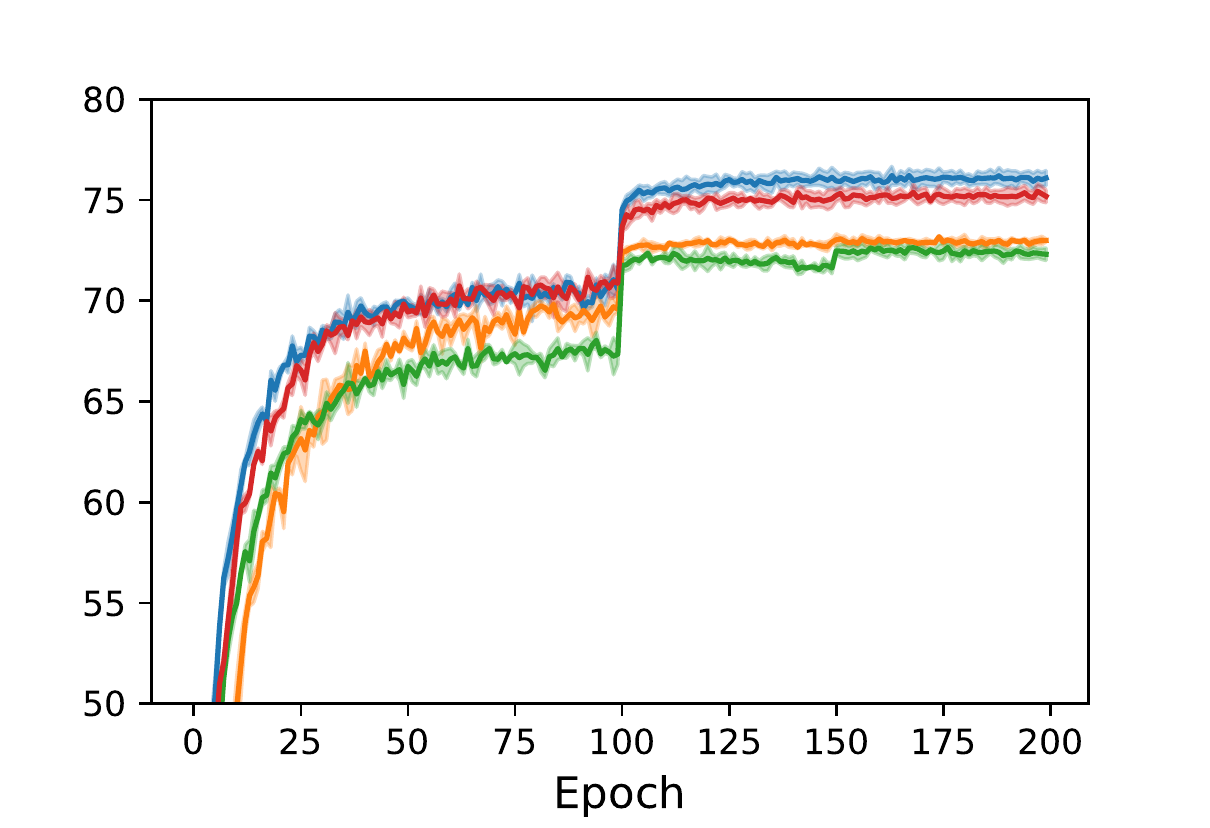}
\end{subfigure}%
\begin{subfigure}{.33\textwidth}
    \centering
    \includegraphics[width=1\textwidth]{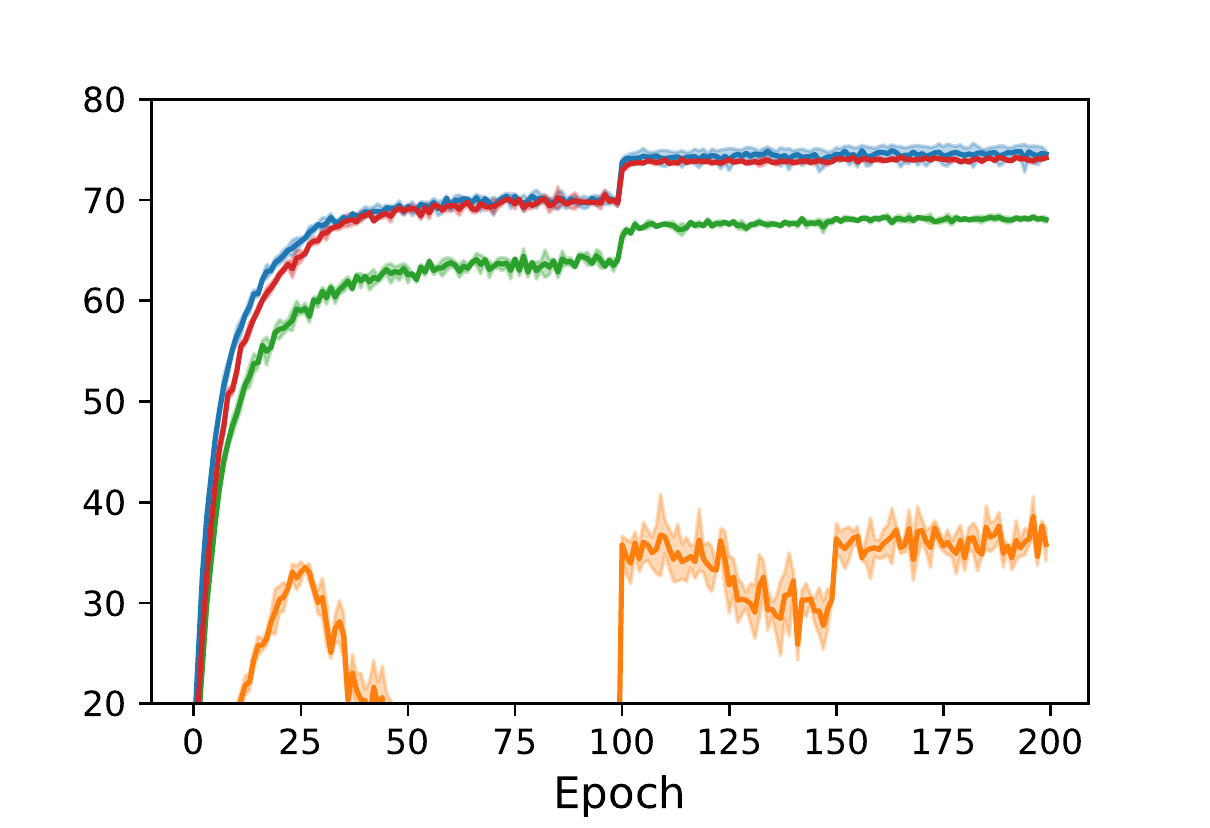}
\end{subfigure}
\caption[short]{Experimental results showing the train and test accuracy percentages on CIFAR-100 using Resnet18 for different batch-sizes. The solid curves represent the mean value and shaded region spans one standard deviation obtained over three replications. Note that the scale of the y-axis varies across the plots. {\esignsgd} consistently and significantly outperforms the other sign-based methods, closely matching the performance of {\sgdm}.}
\label{fig:experiments-main}
\end{figure*}
\vspace{-2mm}
\begin{table}[]
\centering
\setlength{\tabcolsep}{2pt}
{\small
\begin{tabular}{|c|c|c|c|c|}\hline
\textbf{Batch-size} & \textbf{\sgdm} & \textbf{\signsgd}  & \textbf{\signum} & \textbf{\esignsgd} \\ \hline
 128 & {75.35}  & -2.21 & -3.15 & \textbf{-0.92} \\ \hline
 32  & {76.22} & -3.04 & -3.57 & \textbf{-0.79} \\ \hline
 8   & {74.91}  & -36.35 & -6.6 & \textbf{-0.64} \\ \hline
\end{tabular}
}%
\caption[short]{Generalization gap on CIFAR-100 using Resnet18 for different batch-sizes. For {\sgdm} we report the best mean test accuracy percentage, and for the other algorithms we report their difference to the {\sgdm} accuracy (i.e. the generalization gap). {\esignsgd} has a much smaller gap.}\label{tab:test-main}
\end{table}

\subsection{Results}
The results of the experiments for Resnet18 on Cifar100 are shown in Fig.~\ref{fig:experiments-main} and Table~\ref{tab:test-main}. Results for VGG19 on Cifar10 are also similar and can be found in the Appendix. We make four main observations:

\vspace{-2mm}
\paragraph{{\esignsgd} is faster than {\sgdm} on train.} On the train dataset, both the accuracy and the losses (Fig. \ref{fig:resnet-full}) is better for {\esignsgd} than for SGD for all batch-sizes and models (Fig. \ref{fig:vgg-full}). In fact even {\signsgd} is faster than {\sgdm} on the train dataset on VGG19 (Fig. \ref{fig:vgg-full}) for batch-size 128. As we note in Section~\ref{para:faster-than-SGD}, the result that the scaled sign methods are also faster than SGD (and in fact faster than even the without feedback algorithms) overturns previously understood intuition (cf. \cite{kingma2014adam,bernstein2018signsgd}) for why {\signsgd} and other adaptive methods outperform SGD---i.e. restricting the effect of a some `bad' coordinates with high variance may not be the main reason why sign based methods are faster than SGD on train.

\vspace{-2mm}
\paragraph{{\esignsgd} almost matches {\sgdm} on test.} On the test dataset (Table \ref{tab:test-main}), the accuracy and the loss is much closer to SGD than the other sign methods across batch-sizes and models (Tables \ref{tab:test-resnet}, \ref{tab:test-vgg}). The generalization gap (both in accuracy and loss) reduces with decreasing batch-size. We believe this is because the learning-rate was scaled proportional to the batch-size and hence smaller learning-rates lead to smaller generalization gap, as was theoretically noted in Remark \ref{rem:generalization-gap}.

\vspace{-2mm}
\paragraph{{\signsgd} performs poorly for small batch-sizes.} The performance of {\signsgd} is always worse than {\esignsgd} indicating that scaling is insufficient and that error-feedback is crucial for performance. Further all metrics (train and test, loss and accuracy) increasingly become worse as the batch-size decreases indicating that {\signsgd} is indeed a brittle algorithm. In fact for batch-size 8, the algorithm becomes extremely unstable.

\paragraph{{\signum} performs poorly on some datasets and for smaller batch-sizes.} We were surprised that the training performance of {\signum} is significantly worse than even {\signsgd} on CIFAR-100 for batch-sizes 128 and 32. On CIFAR-10, on the other hand, {\signum} manages to be faster than {\sgdm} (though still slower than {\esignsgd}). We believe this may be due to {\signum} being sensitive to the weight-decay parameter as was noted in \cite{bernstein2018signsgd}. We do not tune the weight-decay parameter and leave it to its default value for all methods (including {\esignsgd}). Further the generalization gap of {\signum} gets worse for decreasing batch-sizes with a significant 6.6\% drop in accuracy for batch-size 8.
\vspace{-2mm}

\section{Conclusion}
We study the effect of biased compressors on the convergence and generalization of stochastic gradient algorithms for non-convex optimization. We have shown that biased compressors if naively used can lead to bad generalization, and even non-convergence. We then show that using error-feedback all such adverse effects can be mitigated. Our theory and experiments indicate that using error-feedback, our compressed gradient algorithm {\ecsgd} enjoys the same rate of convergence as original SGD---thereby giving compression \emph{for free}. We believe this should have a large impact in the design of future compressed gradient schemes for distributed and decentralized learning. Further, given the relation between sign-based methods and {\adam}, we believe that our results will be relevant for better understanding the performance and limitations of adaptive methods. Finally, in this work we only consider the single worker case. Developing a practical and scalable algorithm for multiple workers is a fruitful direction for future work.
%


\paragraph*{Acknowledgements.}
We are grateful to Tao Lin, Thijs Vogels, and Negar Foroutan for their help with the experiments. We also thank Jean-Baptiste Cordonnier, Konstantin Mishchenko, Jeremy Bernstein, and anonymous reviewers for their suggestions which helped improve our writeup.

%% file: figs/non_smooth_counter_edit.pdf_tex
\begingroup%
  \makeatletter%
  \providecommand\color[2][]{%
    \errmessage{(Inkscape) Color is used for the text in Inkscape, but the package 'color.sty' is not loaded}%
    \renewcommand\color[2][]{}%
  }%
  \providecommand\transparent[1]{%
    \errmessage{(Inkscape) Transparency is used (non-zero) for the text in Inkscape, but the package 'transparent.sty' is not loaded}%
    \renewcommand\transparent[1]{}%
  }%
  \providecommand\rotatebox[2]{#2}%
  \ifx\svgwidth\undefined%
    \setlength{\unitlength}{252.3833313bp}%
    \ifx\svgscale\undefined%
      \relax%
    \else%
      \setlength{\unitlength}{\unitlength * \real{\svgscale}}%
    \fi%
  \else%
    \setlength{\unitlength}{\svgwidth}%
  \fi%
  \global\let\svgwidth\undefined%
  \global\let\svgscale\undefined%
  \makeatother%
  \begin{picture}(1,0.81524926)%
    \put(0,0){\includegraphics[width=\unitlength,page=1]{../figs/non_smooth_counter_edit.pdf}}%
    \put(0.19892097,0.07920289){\color[rgb]{0,0,0}\makebox(0,0)[lb]{\smash{1}}}%
    \put(0.41489316,0.6338493){\color[rgb]{0,0,0}\rotatebox{-71.13192279}{\makebox(0,0)[lt]{\begin{minipage}{0.26071002\unitlength}\raggedright $\gg_1$\end{minipage}}}}%
    \put(0.72756012,0.27020422){\color[rgb]{0,0,0}\rotatebox{-18.39269304}{\makebox(0,0)[lt]{\begin{minipage}{0.5906131\unitlength}\raggedright $\gg_2$\end{minipage}}}}%
    \put(0.49264623,0.65956652){\color[rgb]{0,0,0}\rotatebox{-43.4575391}{\makebox(0,0)[lt]{\begin{minipage}{0.48342989\unitlength}\raggedright $\ss_1$\end{minipage}}}}%
    \put(0.76477351,0.38606471){\color[rgb]{0,0,0}\rotatebox{-43.4575391}{\makebox(0,0)[lt]{\begin{minipage}{0.48342989\unitlength}\raggedright $\ss_2$\end{minipage}}}}%
    \put(0.54343212,0.43938025){\color[rgb]{0,0,0}\rotatebox{39.05825168}{\makebox(0,0)[lt]{\begin{minipage}{0.50965554\unitlength}\raggedright $\ee_1$\end{minipage}}}}%
    \put(0.58108825,0.40078145){\color[rgb]{0,0,0}\rotatebox{39.05825168}{\makebox(0,0)[lt]{\begin{minipage}{0.50965554\unitlength}\raggedright $\ee_2$\end{minipage}}}}%
  \end{picture}%
\endgroup%

%% file: appendix.tex
\section{Additional Experiments}\label{sec:more-experiments}
In this section we give the full experimental details and results.
\subsection{Convergence under sparse noise}\label{subsec:sparse-noise}
\begin{figure}[!htbp]
	\centering
	\includegraphics[width=0.5\linewidth]{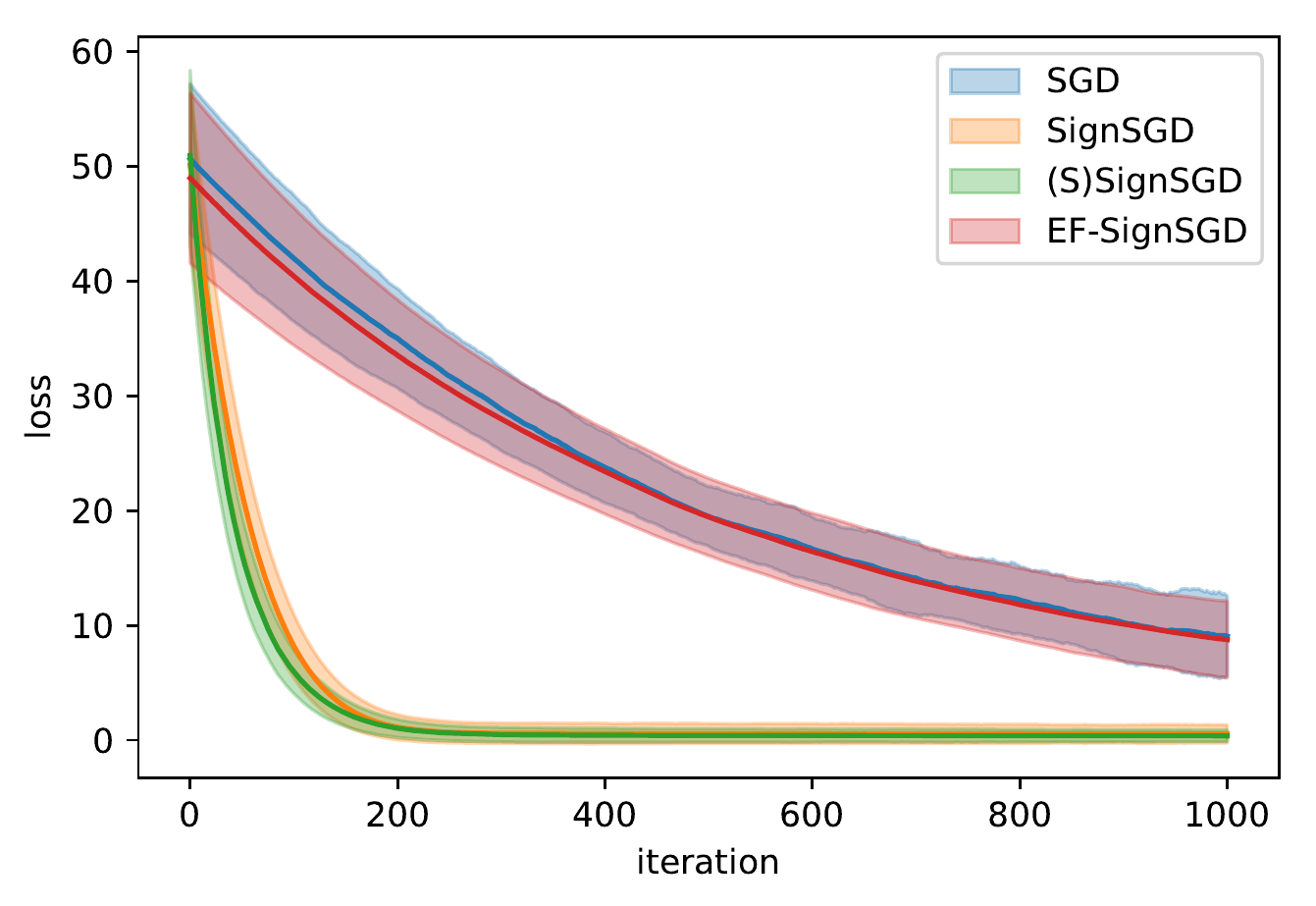}
	\caption{A simple toy problem where {\signsgd} and (scaled){\signsgd} are faster than both SGD and {\esignsgd}. The experiment is repeated 100 times with mean indicated by the solid line and the shaded region spans one standard deviation. As in \cite{bernstein2018signsgd}, the loss is $f(\xx) = \frac{1}{2}\norm{\xx}^2_2$ for $\xx \in \real^{100}$, with gradient $\nabla f(\xx) = \xx$. The stochastic gradient is constructed by adding Gaussian noise $N(0,100^2)$ to only the first coordinate of the gradient. The best learning-rate for SGD and {\esignsgd} was found to be 0.001, and for {\signsgd} and (scaled){\signsgd} was 0.01. The conclusion of this toy experiment directly contradicts the results of our real-world experiments (Section \ref{sec:experiments}) where {\esignsgd} is faster during training than both SGD and {\signsgd}. This shows that the sparse noisy coordinate explanation proposed by \cite{bernstein2018signsgd} is probably an incorrect explanation for the speed of sign based methods during training.}
\end{figure}
\subsection{Description of models and datasets}

\textbf{The cifar dataset.} The CIFAR 10 and 100 training and testing datasets was loaded using the default Pytorch torchvision api\footnote{\url{https://pytorch.org/docs/stable/torchvision/index.html}}. Data augmentation consisting of random $32 \times 32$ crops (padding 4) and horizontal flips was performed. Both sets were normalized over each separate channel.

\textbf{VGG (on CIFAR 10).} We used VGG19 architecture consisting in the following layers: \\
64 -> 64 -> M -> 128 -> 128 -> M -> 256 -> 256 -> 256 -> 256 -> M -> 512 -> 512 -> 512 -> 512 -> M -> 512 -> 512 -> 512 -> 512 -> M\\
where M denotes max pool layers (kernel 2 and stride 2), and each of the number $n$ (either of 64/128/256/512) represents a two dimensional convolution layer with $n$ channels a kernel of 3 and a padding of 1. All of them are followed by a batch normalization layer. Everywhere, ReLU activation is used.

\textbf{Resnet (on CIFAR 100).} We used a standard Resnet18 architecture with one convolution followed by four blocks and one dense layer \footnote{\url{https://github.com/kuangliu/pytorch-cifar/blob/master/models/resnet.py}}.

\subsection{Learning rate tuning}
For all the experiments, the learning rate was divided by 10 at epochs 100 and again at 150. We tuned the initial learning rate on batchsize 128. The learning rates for batchsize 32 and 8 were scaled down by 4 and 16 respectivley. To tune the initial learning rate, the algorithm was run with the same constant learingrate for 100 epochs. Then the learning rate which resulted in the best (i.e. smallest) test loss is chosen. The search space of possible learning rates was taken to be 9 values equally spaced in logarithmic scale over $10^{-5}$ to $10^1$ (inclusive).

The numbers below are rounded values (2 significant digits) of the actual learning rates:
\[
1.0 \times 10^{-5}, 5.6 \times 10^{-5}, 3.2 \times 10^{-4}, 1.8 \times 10^{-3}, {1.0 \times 10^{-2}}, 5.6 \times 10^{-2}, 3.2 \times 10^{-1}, 1.8 \times 10^0, 1.0 \times 10^1\,.
\]
 The best learning rate for each of the method is shown in table \ref{tab:learning-rate}.
 \begin{table}[!htbp]
 \centering
 \begin{tabular}{|c|c|c|}\hline
 \textbf{Algorithm} & \textbf{Resnet18} & \textbf{VGG19} \\\hline
    \sgdm & $1.0\times 10^{-2}$ & $1.0\times 10^{-2}$ \\\hline
    {\signsgd} & $5.6\times 10^{-2}$ & $5.6 \times 10^{-2}$  \\\hline
    {\signum} & $3.2 \times 10^{-4}$ & $5.6\times 10^{-5}$\\\hline
    {\esignsgd} & $5.6\times 10^{-2}$ & $5.6\times 10^{-2}$\\\hline
\end{tabular}
\caption{The best initial learning rates for the four algorithms for batch size 128 on VGG19 (CIFAR 10 data) and Resnet18 (CIFAR 100 data).}
\label{tab:learning-rate}
\end{table}

\subsection{Experiments with Resnet}
We report the complete results (including the losses) for Resnet in Fig.~\ref{fig:resnet-full}.
\begin{figure*}[!htbp]
\centering
\begin{subfigure}{0.33\textwidth}
    \centering
    \includegraphics[width=1\textwidth]{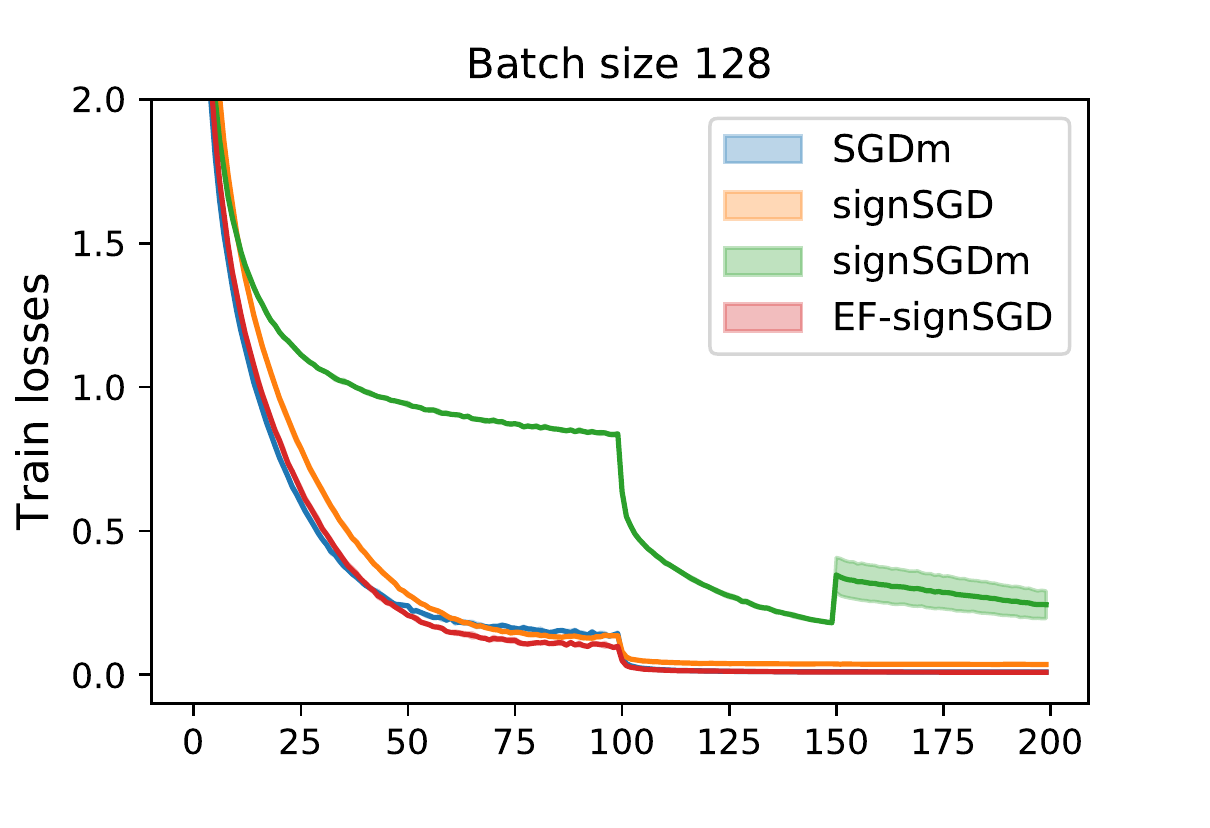}
\end{subfigure}%
\begin{subfigure}{0.33\textwidth}
    \centering
    \includegraphics[width=1\textwidth]{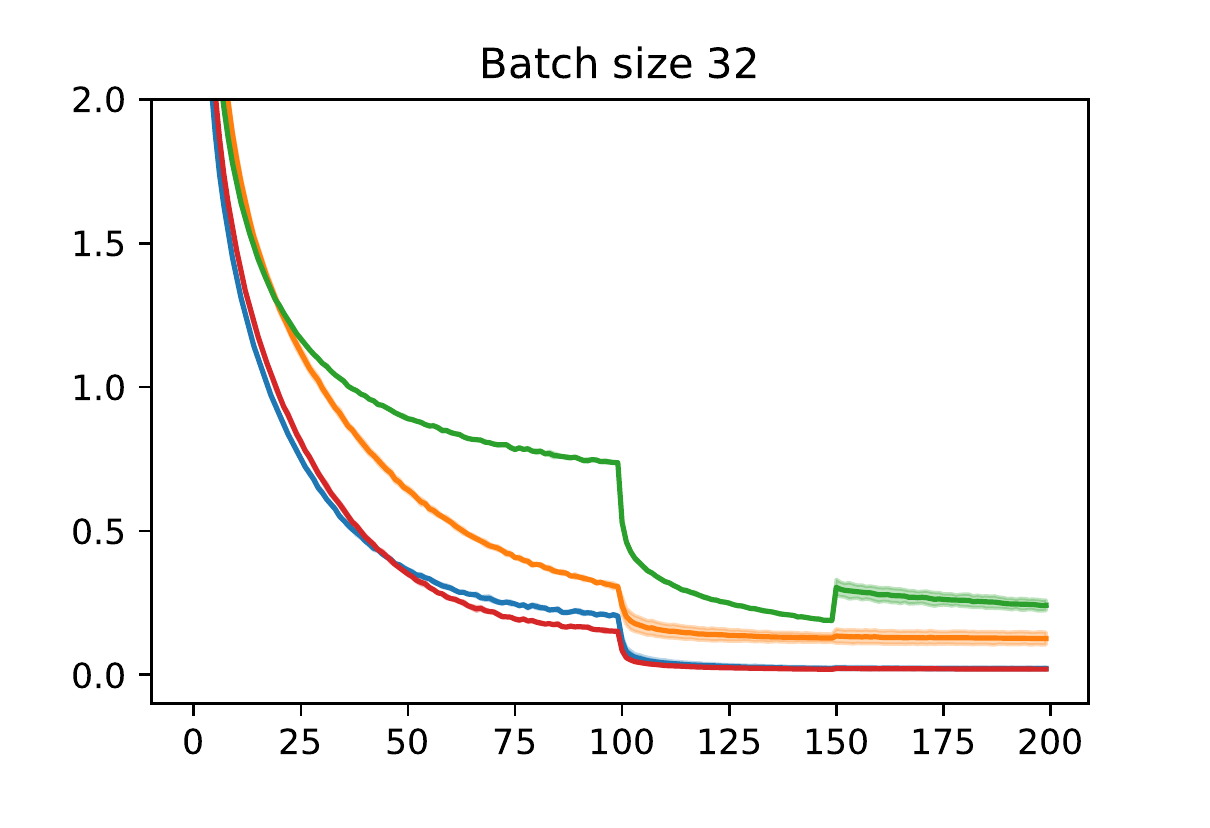}
\end{subfigure}%
\begin{subfigure}{0.33\textwidth}
    \centering
    \includegraphics[width=1\textwidth]{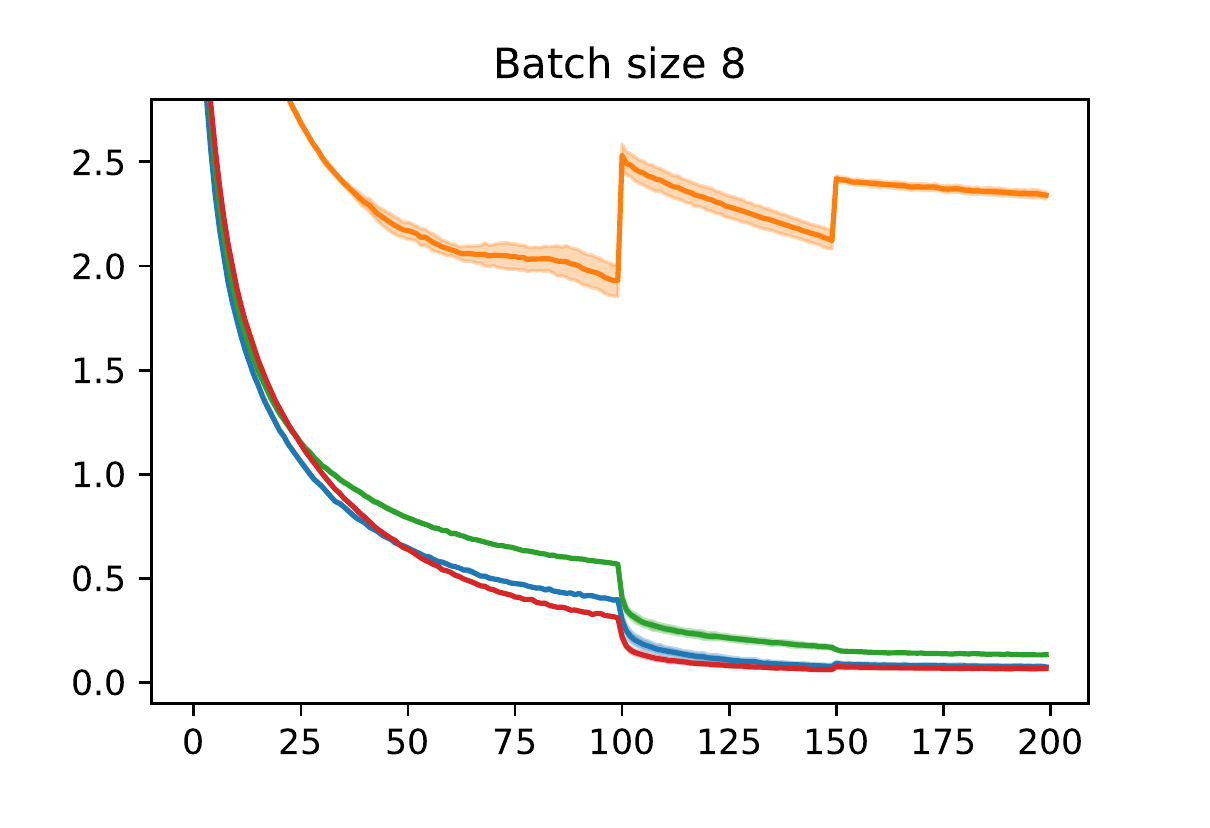}
\end{subfigure}

\begin{subfigure}{0.33\textwidth}
    \centering
    \includegraphics[width=1\textwidth]{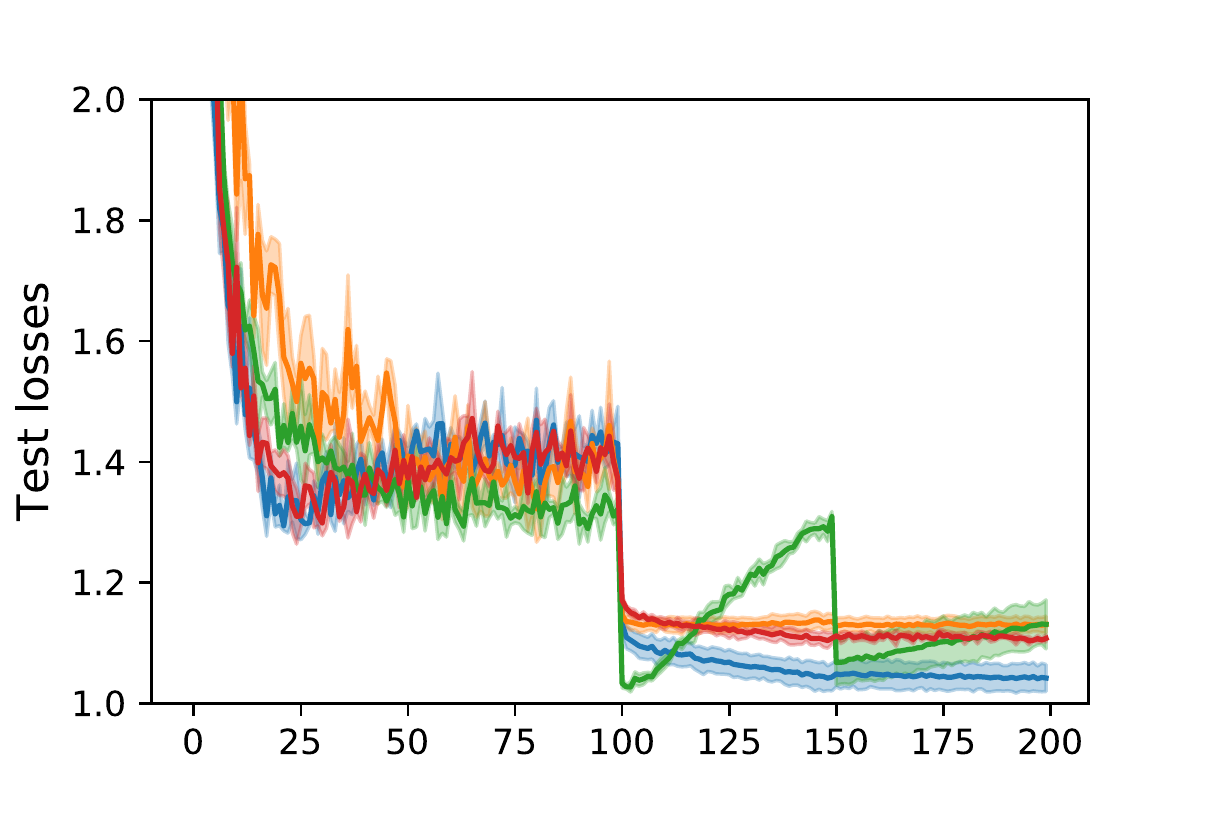}
\end{subfigure}%
\begin{subfigure}{0.33\textwidth}
    \centering
    \includegraphics[width=1\textwidth]{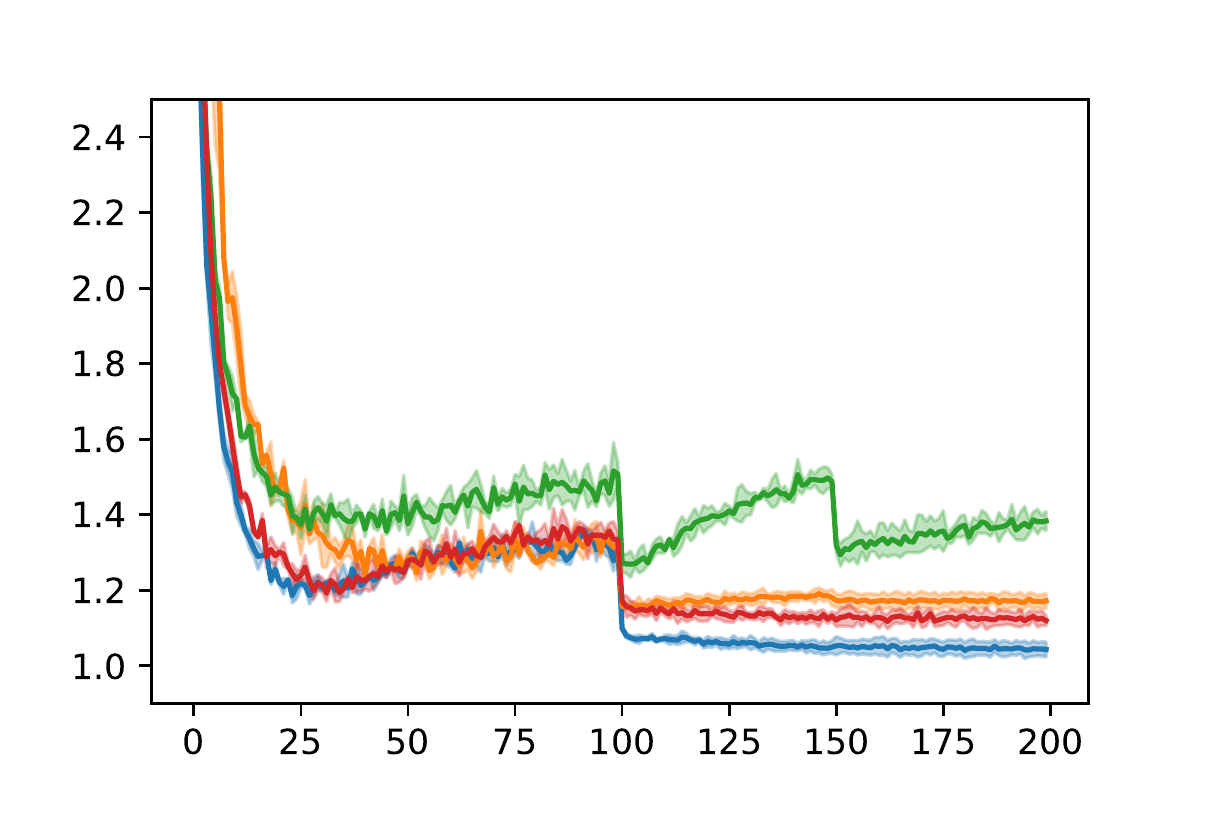}
\end{subfigure}%
\begin{subfigure}{0.33\textwidth}
    \centering
    \includegraphics[width=1\textwidth]{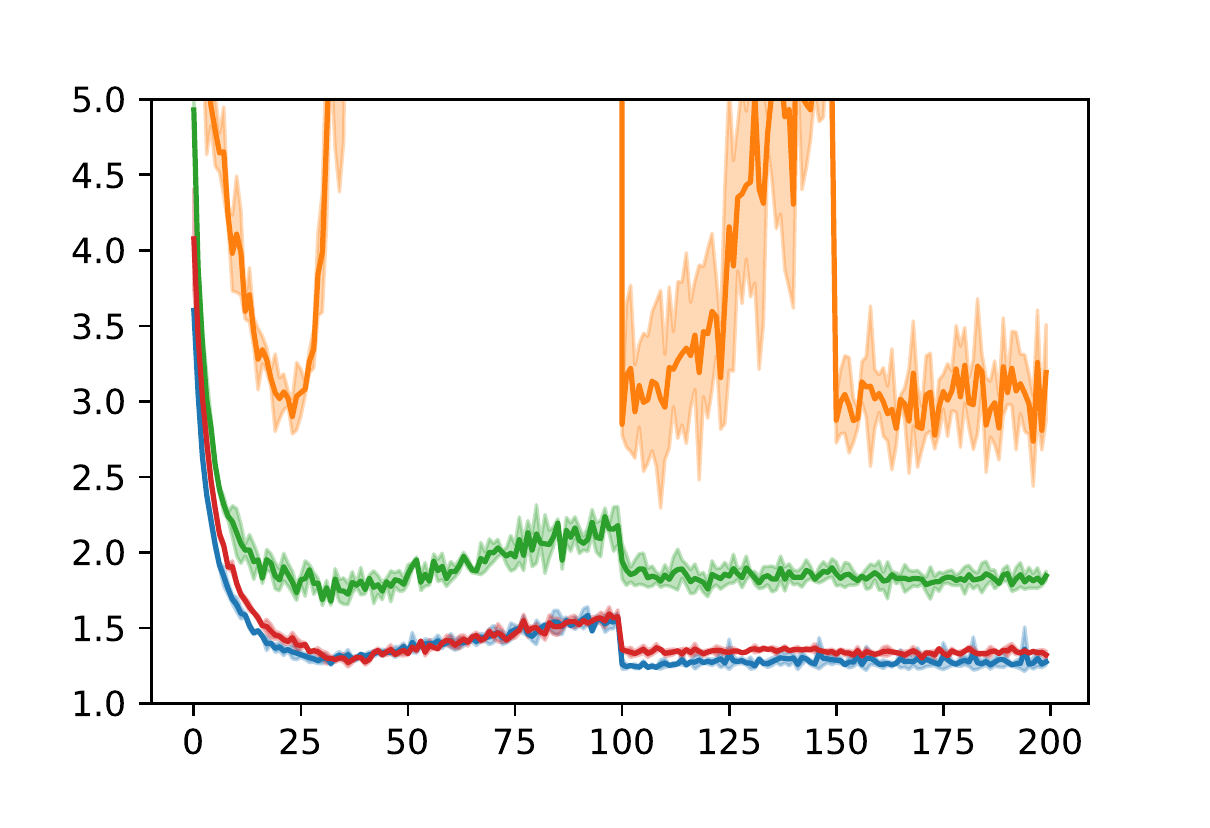}
\end{subfigure}

\begin{subfigure}{0.33\textwidth}
    \centering
    \includegraphics[width=1\textwidth]{figs/main_experiments/batchsize-128/resnet-cifar100/train_accuracies.pdf}
\end{subfigure}%
\begin{subfigure}{0.33\textwidth}
    \centering
    \includegraphics[width=1\textwidth]{figs/main_experiments/batchsize-32/resnet-cifar100/train_accuracies.pdf}
\end{subfigure}%
\begin{subfigure}{0.33\textwidth}
    \centering
    \includegraphics[width=1\textwidth]{figs/main_experiments/batchsize-8/resnet-cifar100/train_accuracies.pdf}
\end{subfigure}

\begin{subfigure}{.33\textwidth}
    \centering
    \includegraphics[width=1\textwidth]{figs/main_experiments/batchsize-128/resnet-cifar100/test_accuracies.pdf}
\end{subfigure}%
\begin{subfigure}{.33\textwidth}
    \centering
    \includegraphics[width=1\textwidth]{figs/main_experiments/batchsize-32/resnet-cifar100/test_accuracies.pdf}
\end{subfigure}%
\begin{subfigure}{.33\textwidth}
    \centering
    \includegraphics[width=1\textwidth]{figs/main_experiments/batchsize-8/resnet-cifar100/test_accuracies.pdf}
\end{subfigure}
\caption[short]{Experimental results showing the loss values and accuracy percentages on the train and test datasets, on CIFAR-100 using Resnet18 for different batch-sizes. The solid curves represent the mean value and shaded region spans one standard deviation obtained over three repetitions. Note that the scale of the y-axis varies across the plots. The losses behave very similar to the accuracies---{\esignsgd} consistently and significantly outperforms the other sign-based methods, is faster than {\sgdm} on train, and closely matches {\sgdm} on test.}\label{fig:resnet-full}
\end{figure*}

\begin{table}[!htbp]
\centering
\begin{tabular}{cc|c|c|c|c|}
\cline{3-6}
\multicolumn{1}{l}{}                                                                        &     & \multicolumn{4}{c|}{Algorithm} \\ \cline{3-6}
\multicolumn{1}{l}{}                                                                        &     & \textbf{\sgdm} & \textbf{scaled \signsgd} & \textbf{\signum} & \textbf{\esignsgd} \\ \hline
\multicolumn{1}{|c|}{\multirow{3}{*}{\begin{tabular}[c]{@{}c@{}}Batch\\ size\end{tabular}}}
& 128 & {75.35}  & -2.21 & -3.15 & \textbf{-0.92} \\ \cline{2-6}
\multicolumn{1}{|c|}{}
& 32  & {76.22} & -3.04 & -3.57 & \textbf{-0.79} \\ \cline{2-6}
\multicolumn{1}{|c|}{}
& 8   & {74.91} & -36.35 & -6.6 & \textbf{-0.64} \\ \hline
\end{tabular}
\caption[short]{Generalization gap on CIFAR-100 using Resnet18 for different batch-sizes. For {\sgdm} we report the best mean test accuracy percentage, and for the other algorithms we report their difference to the {\sgdm} accuracy (i.e. the generalization gap). {\esignsgd} has a much smaller gap which decreases with decreasing batchsize. The generalization gap of {\signum} and {\signsgd} increases as the batchsize decreases.}\label{tab:test-resnet}
\end{table}

\subsection{Experiments with VGG}
We report the complete results (including the losses) for VGG in Fig.~\ref{fig:vgg-full}.
\begin{figure*}[!htbp]
\centering

\begin{subfigure}{0.33\textwidth}
    \centering
    \includegraphics[width=1\textwidth]{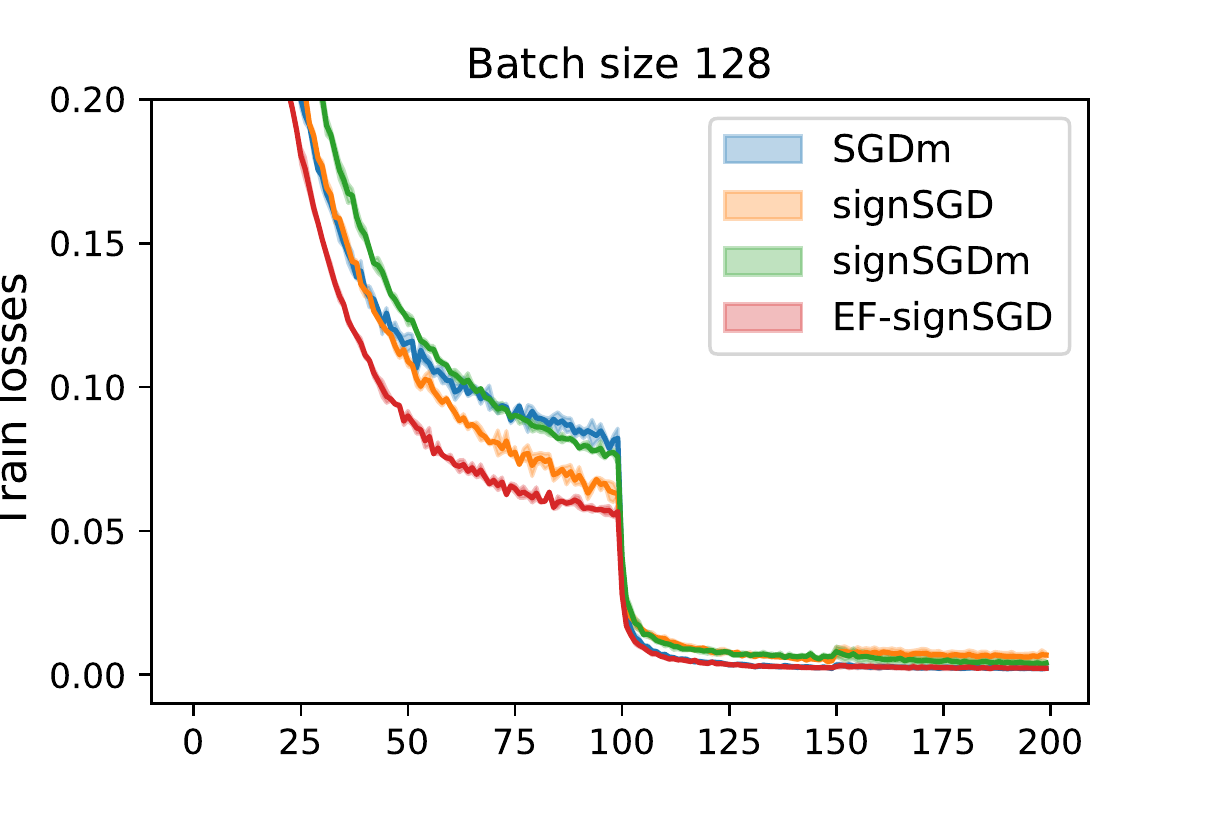}
\end{subfigure}%
\begin{subfigure}{0.33\textwidth}
    \centering
    \includegraphics[width=1\textwidth]{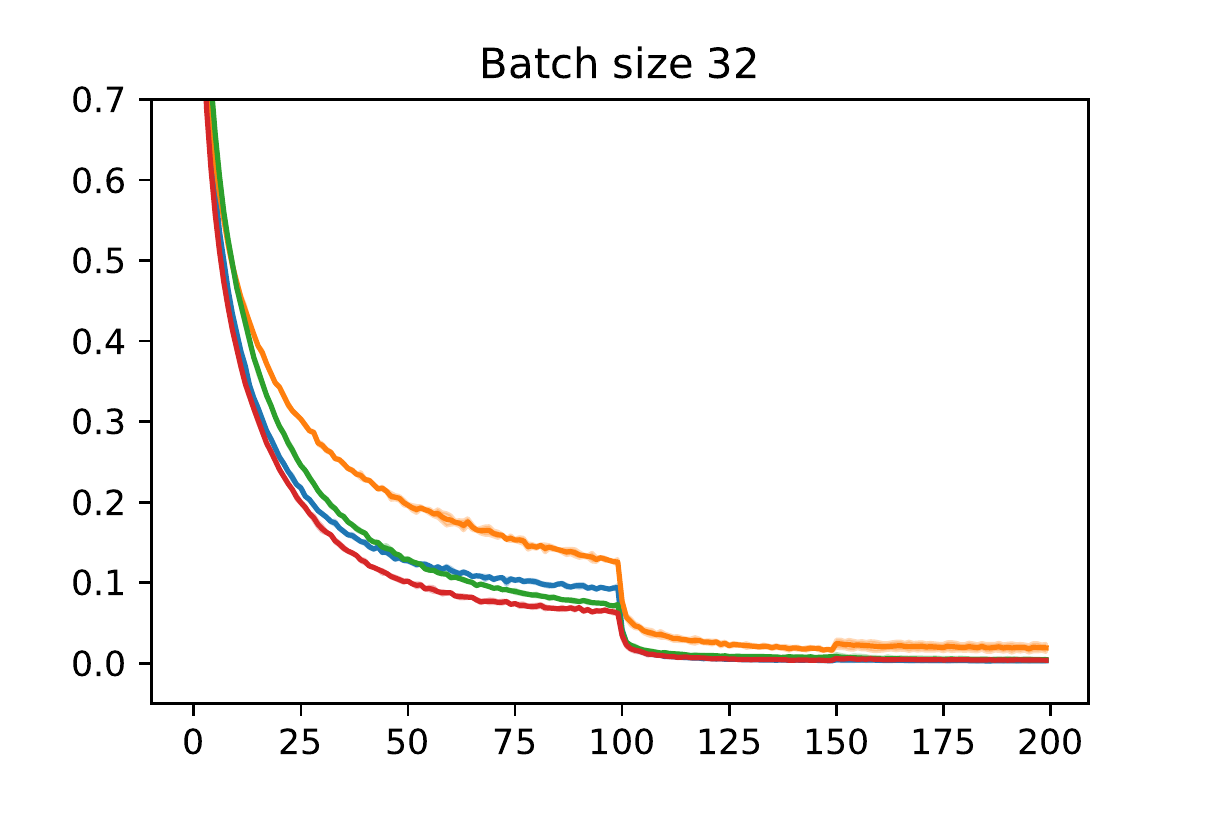}
\end{subfigure}%
\begin{subfigure}{0.33\textwidth}
    \centering
    \includegraphics[width=1\textwidth]{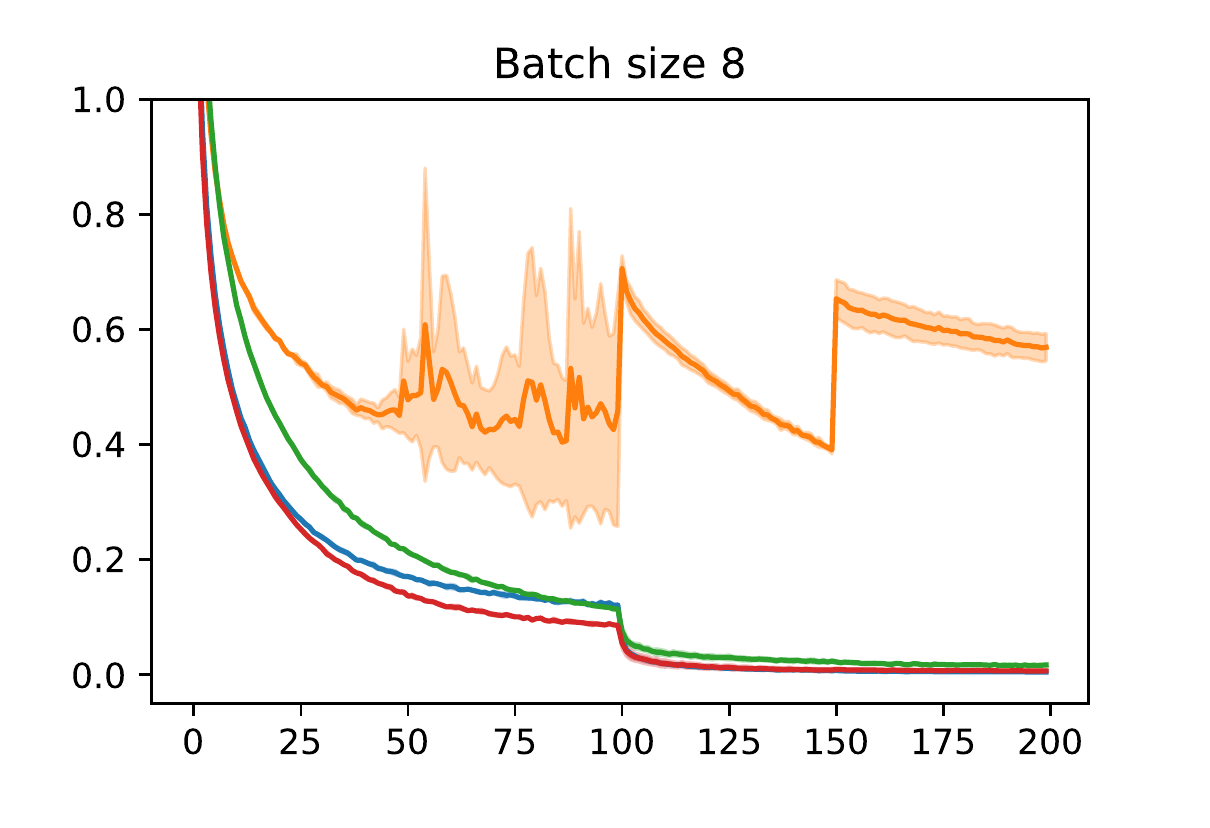}
\end{subfigure}

\begin{subfigure}{0.33\textwidth}
    \centering
    \includegraphics[width=1\textwidth]{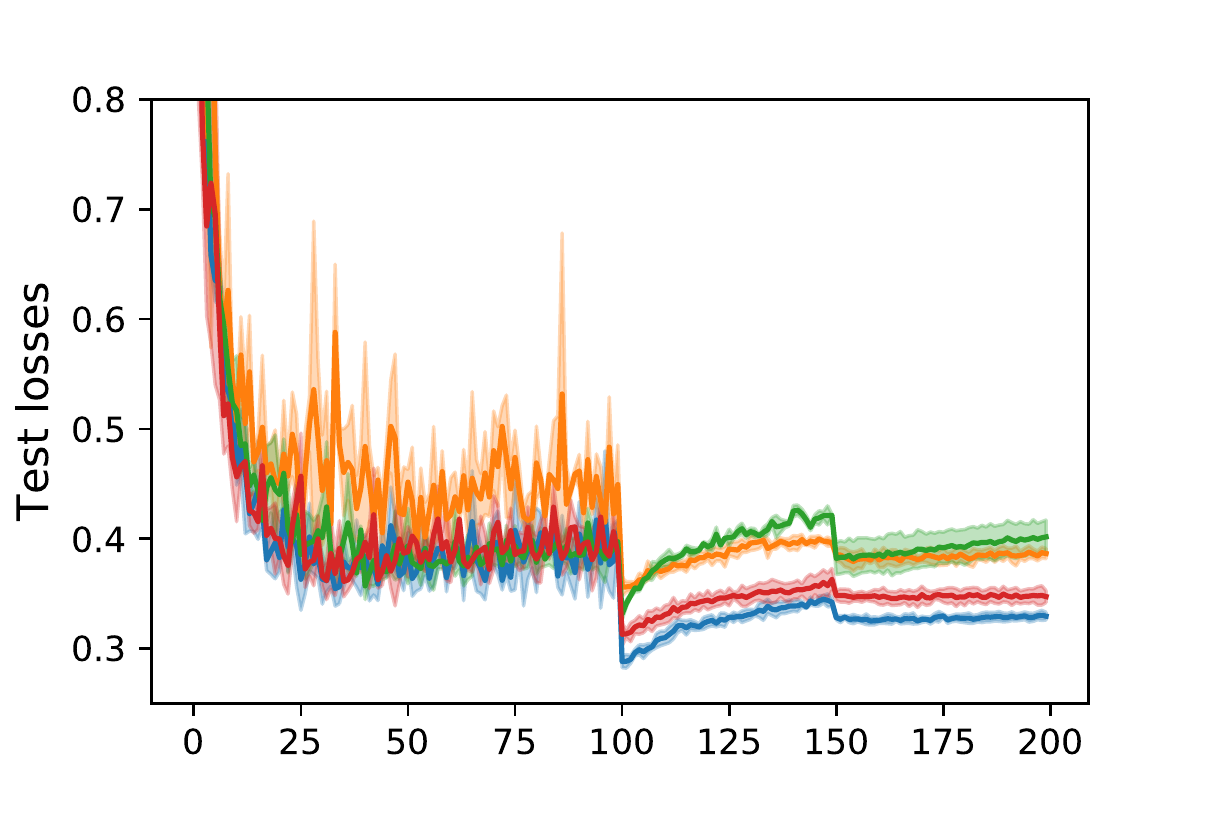}
\end{subfigure}%
\begin{subfigure}{0.33\textwidth}
    \centering
    \includegraphics[width=1\textwidth]{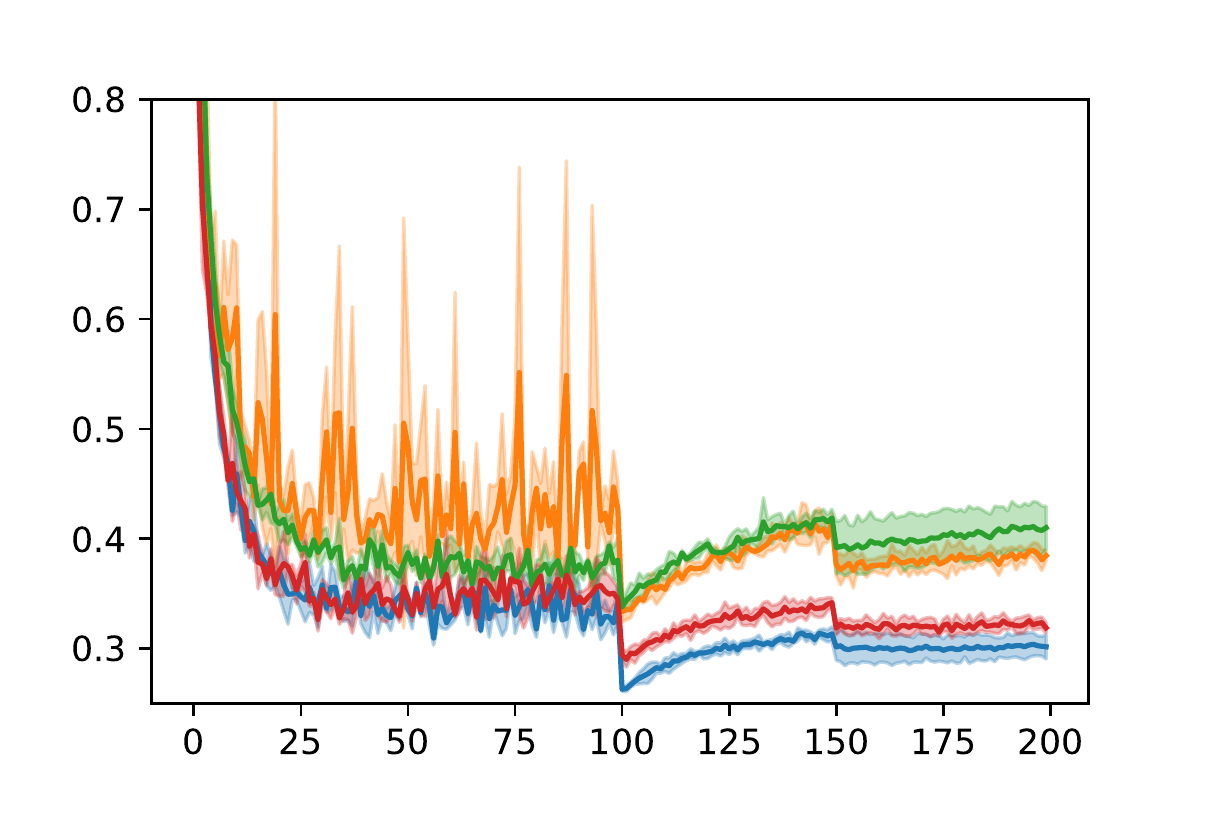}
\end{subfigure}%
\begin{subfigure}{0.33\textwidth}
    \centering
    \includegraphics[width=1\textwidth]{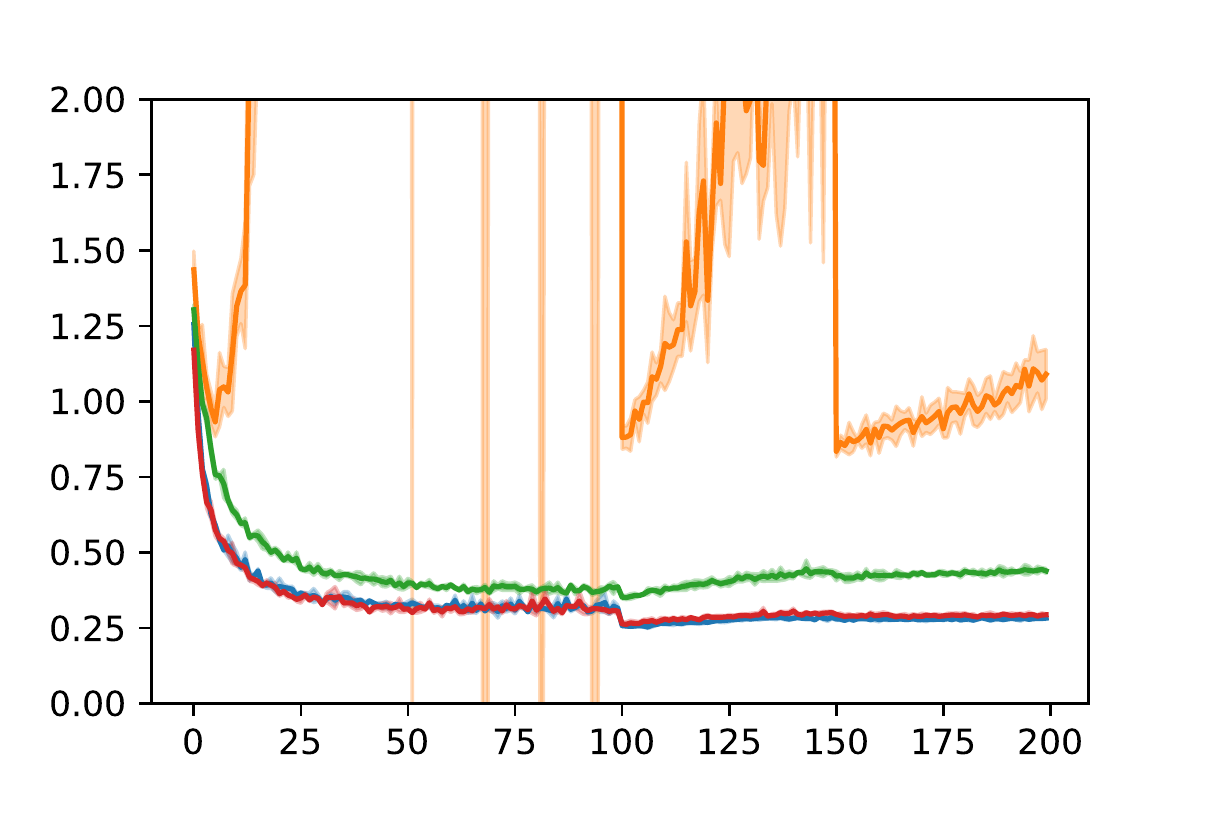}
\end{subfigure}

\begin{subfigure}{0.33\textwidth}
    \centering
    \includegraphics[width=1\textwidth]{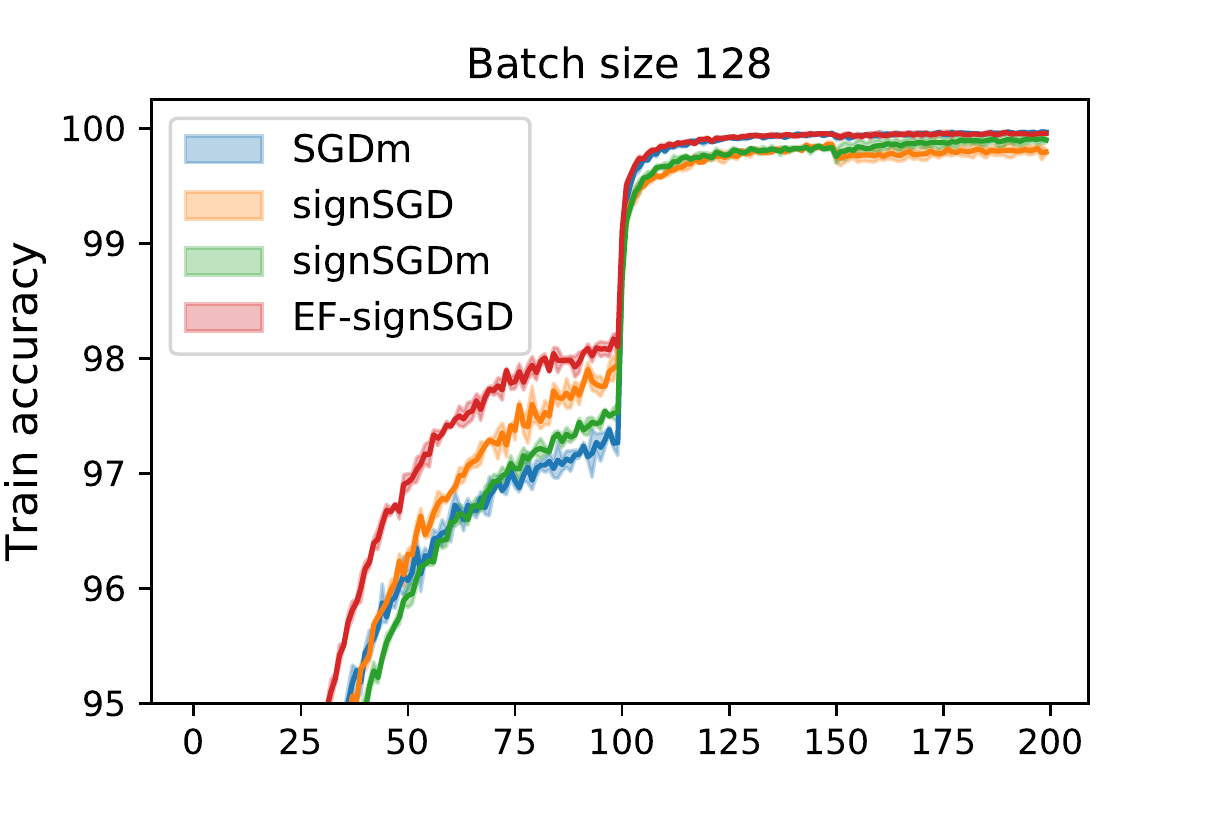}
\end{subfigure}%
\begin{subfigure}{0.33\textwidth}
    \centering
    \includegraphics[width=1\textwidth]{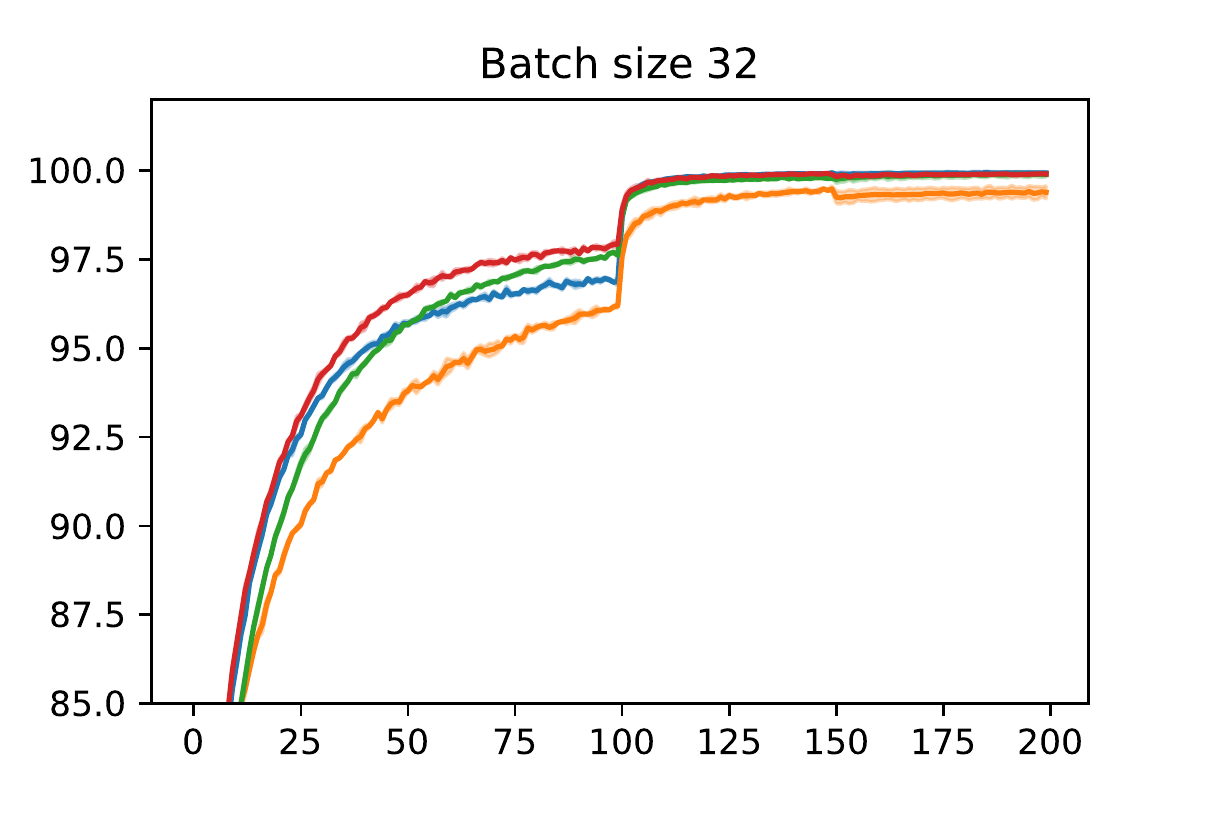}
\end{subfigure}%
\begin{subfigure}{0.33\textwidth}
    \centering
    \includegraphics[width=1\textwidth]{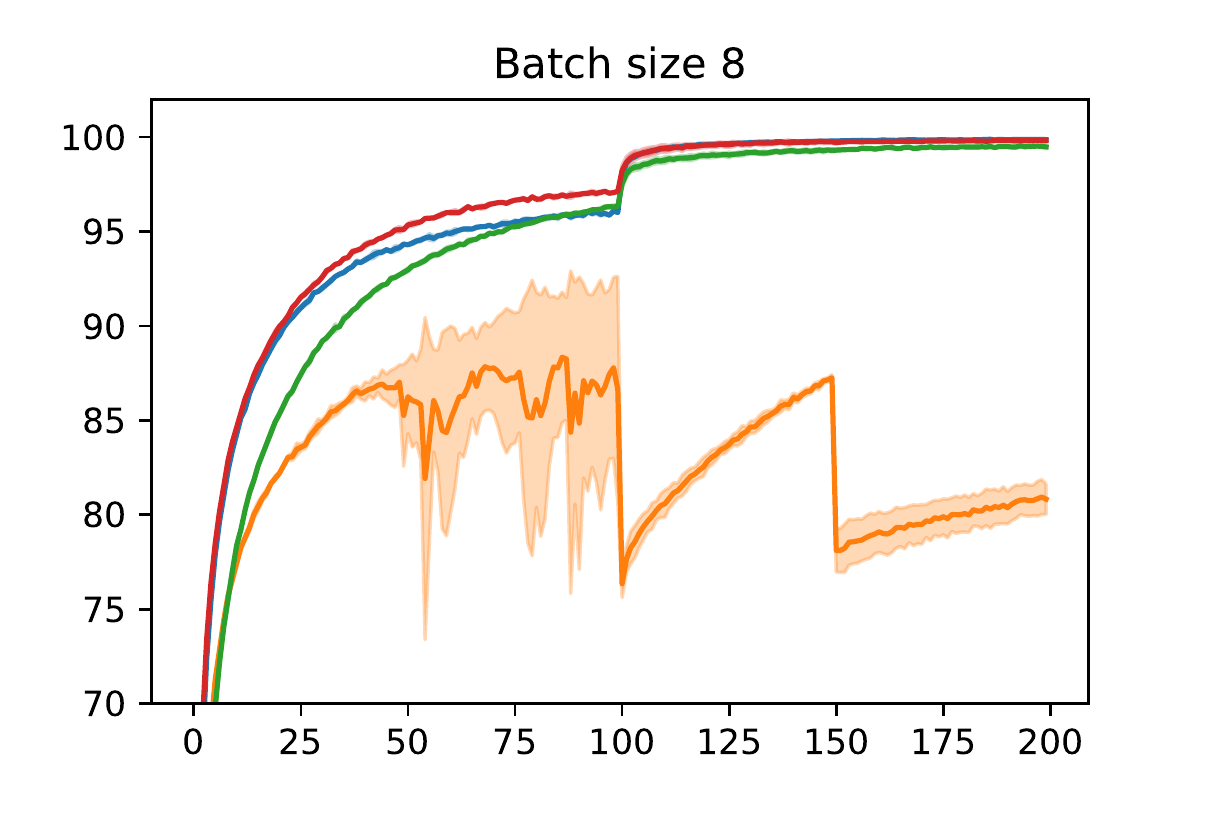}
\end{subfigure}

\begin{subfigure}{.33\textwidth}
    \centering
    \includegraphics[width=1\textwidth]{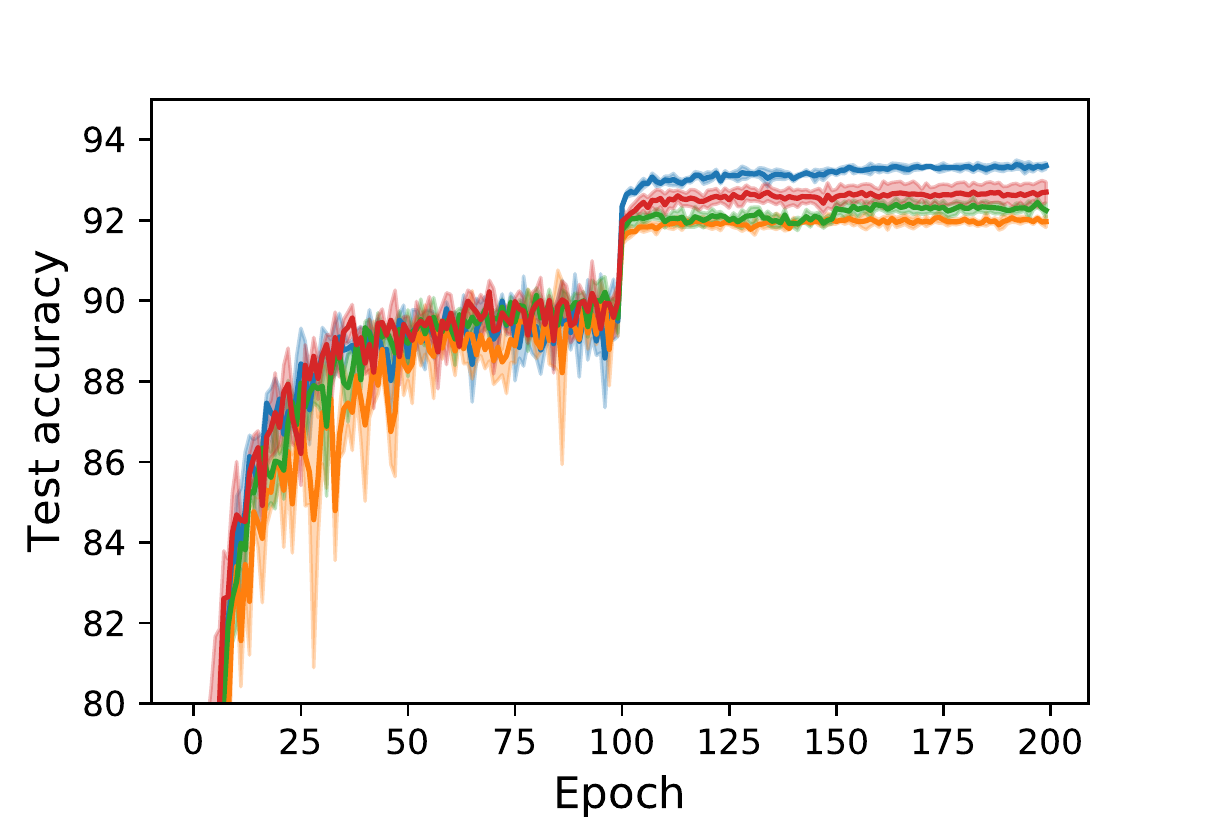}
\end{subfigure}%
\begin{subfigure}{.33\textwidth}
    \centering
    \includegraphics[width=1\textwidth]{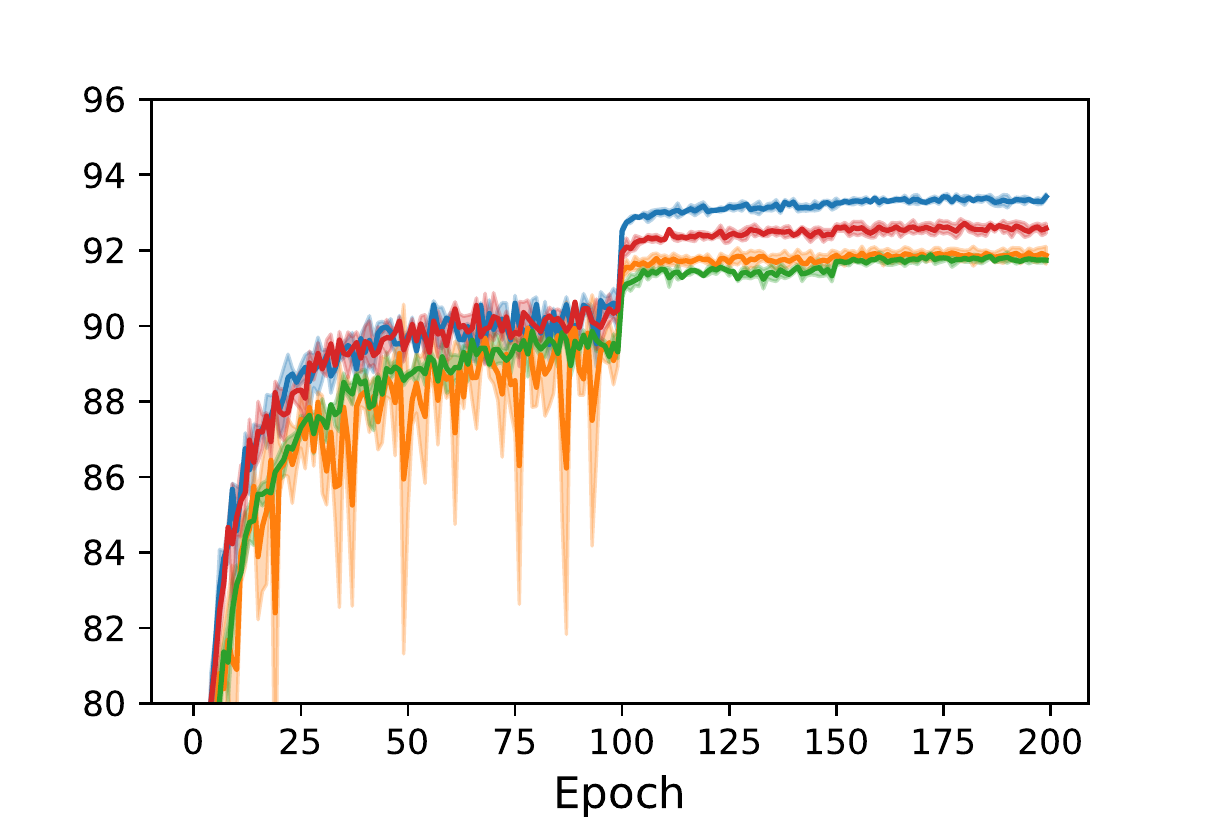}
\end{subfigure}%
\begin{subfigure}{.33\textwidth}
    \centering
    \includegraphics[width=1\textwidth]{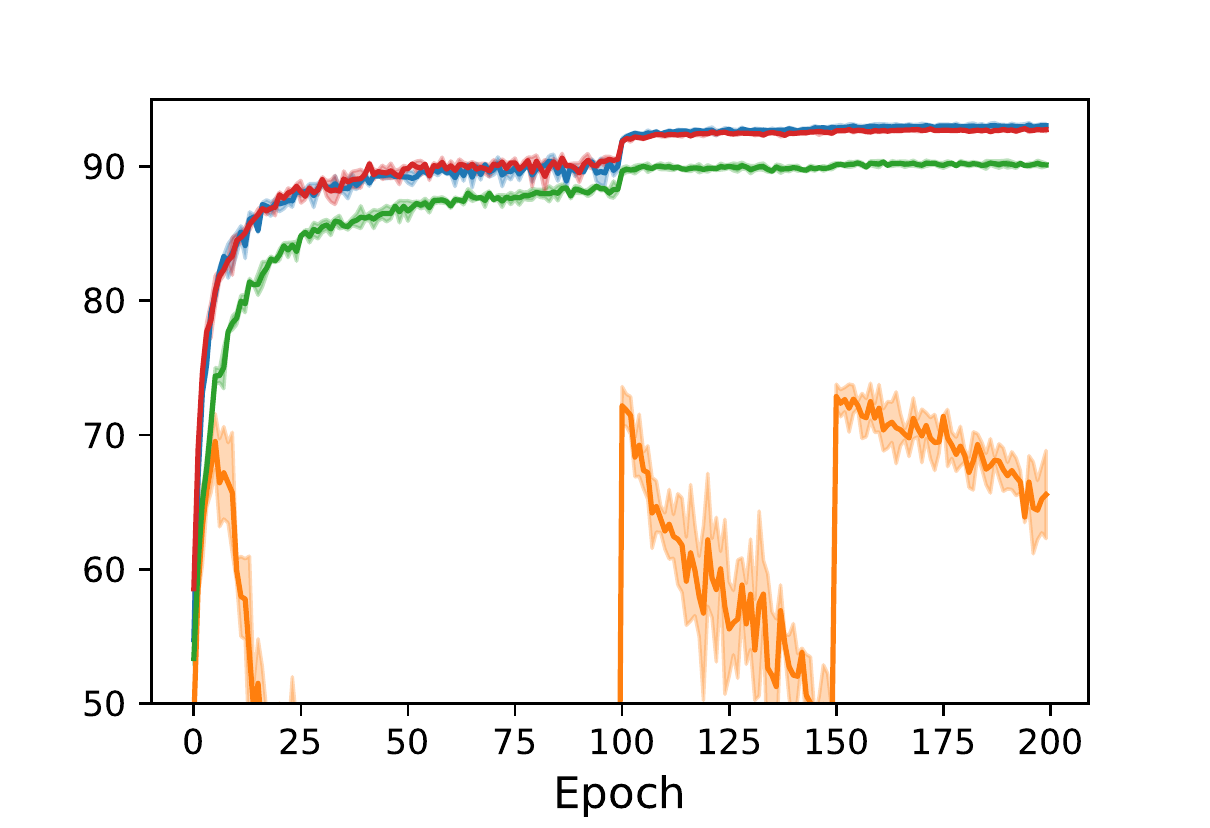}
\end{subfigure}

\caption[short]{Experimental results showing the loss values and accuracy percentages on the train and test datasets, on CIFAR-10 using VGG19 for different batch-sizes. The solid curves represent the mean value and shaded region spans one standard deviation obtained over three repetitions. Note that the scale of the y-axis varies across the plots. The plots for VGG19 behave very similarly to that of Resnet18, except that {\signum} performs better on the train dataset. On the test dataset, {\signum} and the other algorithms behave exactly as in Resnet. Here too, {\esignsgd} consistently and significantly outperforms the other sign-based methods, is faster than {\sgdm} on train, and also closely matches the test performance of {\sgdm}.}
\label{fig:vgg-full}
\end{figure*}

\begin{table}[!htbp]
\centering
\begin{tabular}{cc|c|c|c|c|}
\cline{3-6}
\multicolumn{1}{l}{}                                                                        &     & \multicolumn{4}{c|}{Algorithm} \\ \cline{3-6}
\multicolumn{1}{l}{}                                                                        &     &
\textbf{\sgdm} & \textbf{scaled \signsgd} & \textbf{\signum} & \textbf{\esignsgd} \\ \hline
\multicolumn{1}{|c|}{\multirow{3}{*}{\begin{tabular}[c]{@{}c@{}}Batch\\ size\end{tabular}}}
& 128 & {93.38} & -1.31 & -0.94 & \textbf{-0.68} \\ \cline{2-6}
\multicolumn{1}{|c|}{}
& 32  & {93.42} & -1.49 & -1.54 & \textbf{-0.71} \\ \cline{2-6}
\multicolumn{1}{|c|}{}
& 8   & {93.09} & -20.22 & -2.75 & \textbf{-0.27} \\ \hline
\end{tabular}
\caption[short]{Generalization gap on CIFAR-10 using VGG19 for different batch-sizes. For {\sgdm} we report the best mean test accuracy percentage, and for the other algorithms we report their difference to the {\sgdm} accuracy (i.e. the generalization gap). {\esignsgd} has a much smaller gap which decreases with decreasing batchsize. The generalization gap of {\signum} and {\signsgd} increases as the batchsize decreases.}\label{tab:test-vgg}
\end{table}

\newpage

\subsection{Data generation process (Section \ref{subesc:linear-span-simulations})}
The data is generated as in Section 3.3 of \cite{wilson2017marginal}. We fix $n = 200$ (the number of data points) and $d= 6n$ (dimension) in the below process. Each entry of the target label vector $\yy \in \{-1,1\}^n$ is uniformly set as $-1$ or $1$. Then the $j$th coordinate (column) of the $i$th data point (row) in the data matrix $\mA \in \real^{n\times d}$ is filled as follows:
\[
  A_{i,j} = \begin{cases}
      y_i & \quad j=1\,,\\
      1 & \quad j=2,3\,,\\
      1 & \quad j = 4 + 5(i − 1),\dots, 4 + 5(i − 1) + 2(1 − y_i)\,,\\
      0 & \quad \text{otherwise}\,.
    \end{cases}
\]
Then the data matrix $\mA$ and labels $\yy$ are randomly (and equally) split between the train and the test dataset. Hence there are 100 data points each of dimension 1200 in the test and train.

\section{Missing Proofs}
In this section we fill out the proofs of claims, lemmas, and theorems made in the main paper.
\subsection{Proof of counter-example (Theorem \ref{thm:counter})}
The stochastic gradient at iteration $t$ is of the form 
\[
  \gg_t = \aa_{i_t} l_{i_t}'(\lin*{\aa_{i_t}}{\xx})\,.
\]
This means that
\[
  \sign{\gg_t} = \sign{\aa_{i_t}}\cdot\sign{l'_{i_t}(\lin*{\aa_{i_t}}{\xx})} = \pm \ss\,.
\]
Thus $\xx_{t+1} = \xx_t \pm \gamma \ss$ and the iterates of {\signsgd} can only move along the direction $\ss$. Then, {\signsgd} can converge only if there exists $\gamma^\star \in \real$ such that
\[
  \xx_0 = \xx^\star + \gamma^\star \ss\,.
\]
Since the measure of this set in $\real^d$ for $d \geq 2$ is 0, we can conclude that {\signsgd} will not converge to $\xx^\star$ almost surely. %

\subsection{Proof of bounded error (Lemma \ref{lem:bounded-error})}
By definition of the error sequence,
\[
  \norm{\ee_{t+1}}^2 = \norm{\cC(\pp_t) - \pp_t}^2_2 \leq (1- \delta)\norm{\pp_t}^2_2 = (1- \delta)\norm{\ee_t + \gamma \gg_t}^2_2\,.
\]
In the inequality above we used that $\cC(\cdot)$ is a $\delta$-approximate compressor. We thus have a recurrence relation on the bound of $\ee_t$. Using Young's ineuqality, we have that for any $\eta > 0$:
\[
  \norm{\ee_{t+1}}^2 \leq (1- \delta)\norm{\ee_t + \gamma \gg_t}^2_2 \leq (1- \delta)(1 + \eta)\norm{\ee_t}^2_2 + \gamma^2 (1- \delta)(1 + 1/\eta)\norm{\gg_t}^2_2\,.
\]
Here on is simple algebraic computations to solve the recurrence relation above:
\begin{align*}
  \expect \norm{\ee_{t+1}}^2 &\leq (1- \delta)(1 + \eta)\expect \norm{\ee_t}^2_2 + \gamma^2 (1- \delta)(1 + 1/\eta)\expect \norm{\gg_t}^2_2\\
  &\leq  \sum_{i=0}^t [(1- \delta)(1 + \eta)]^{t-i}\gamma^2 (1- \delta)(1 + 1/\eta)\expect \norm{\gg_i}^2_2\\
  &\leq \sum_{i=0}^\infty [(1- \delta)(1 + \eta)]^{i}\gamma^2 (1- \delta)(1 + 1/\eta)\sigma^2\\
  &= \frac{(1- \delta)(1 + 1/\eta)}{1 - (1- \delta)(1 + \eta)}\gamma^2\sigma^2 \\
  &= \frac{(1- \delta)(1 + 1/\eta)}{\delta - \eta(1- \delta)}\gamma^2\sigma^2\\
  &= \frac{2(1- \delta)(1 + 1/\eta)}{\delta}\gamma^2\sigma^2\,.
\end{align*}
Let us pick $\eta = \frac{\delta}{2(1 - \delta)}$ such that $1 + 1/\eta =(2 - \delta)/\delta \leq 2/\delta$. Plugging this in the above gives
\begin{align*}
  \expect \norm{\ee_{t+1}}^2 &\leq \frac{2(1- \delta)(1 + 1/\eta)}{\delta}\gamma^2\sigma^2 \leq \frac{4(1-\delta)}{\delta^2} \gamma^2 \sigma^2\,. \tag*{\qed}
\end{align*}

\subsection{Proof of non-convex convergence of {\esignsgd} (Theorem \ref{thm:ecsgd-non-convex})}
As outlined in the proof sketch, the analysis considers the actual sequence $\{\xx_t\}$ as an approximation to the sequence $\{\txx_t\}$, where $\txx_t = \xx_t - \ee_t$.
It satisfies the recurence
\[
  \txx_{t+1} = \xx_{t} - \ee_{t+1} - \cC(\pp_t) = \xx_t - \pp_t = \txx_t - \gamma \gg_t\,.
\]
Since the function $f$ is $L$-smooth,
\begin{align*}
  \expect_t\sbr*{f(\txx_{t+1})}
  &\leq f(\txx_t) + \lin*{\nabla f(\txx_t)}{\expect_t\sbr*{\txx_{t+1} - \txx_t}} + \frac{L}{2}\expect_t\sbr*{\norm*{\txx_{t+1} - \txx_t}^2_2}\\
  &= f(\txx_t) - \gamma\lin*{\nabla f(\txx_t)}{\expect_t\sbr*{\gg_t}} + \frac{L\gamma^2}{2}\expect_t\sbr*{\norm*{\gg_t}^2_2}\\
  &\leq f(\txx_t) - \gamma\lin*{\nabla f(\txx_t)}{\nabla f(\xx_t)} + \frac{L\gamma^2\sigma^2}{2}\,.
\end{align*}
In the above we need to get rid of $\nabla f(\txx_t)$ since we never encounter it in the algorithm. We can do so using an alternate definition of smoothness of $f$:
\[
  \norm{\nabla f(\xx) - \nabla f(\yy)}_2 \leq L\norm{\xx - \yy}_2\,.
\]
Using the above with $\xx = \xx_t$ and $\yy = \txx_t$ we continue as
\begin{align*}
  \expect_t\sbr*{f(\txx_{t+1})}
  &\leq f(\txx_t) - \gamma\lin*{\nabla f(\xx_t)}{\nabla f(\xx_t)} + \frac{L\gamma^2\sigma^2}{2} + \gamma\lin*{\nabla f(\xx_t) - \nabla f(\txx_t)}{\nabla f(\xx_t)}\\
  &\leq f(\txx_t) - \gamma\norm{\nabla f(\xx_t)}^2_2 + \frac{L\gamma^2\sigma^2}{2} + \frac{\gamma \rho}{2}\norm{\nabla f(\xx_t)}^2_2 + \frac{\gamma}{2 \rho}\norm{\nabla f(\xx_t) - \nabla f(\txx_t)}^2_2\\
  &\leq f(\txx_t) - \gamma\norm{\nabla f(\xx_t)}^2_2 + \frac{L\gamma^2\sigma^2}{2} + \frac{\gamma \rho}{2}\norm{\nabla f(\xx_t)}^2_2 + \frac{\gamma L^2}{2 \rho}\norm{\xx_t - \txx_t}^2_2\\
  &\leq f(\txx_t) - \gamma \rbr*{1- \frac{\rho}{2}}\norm{\nabla f(\xx_t)}^2_2 + \frac{L\gamma^2\sigma^2}{2}+ \frac{\gamma L^2}{2 \rho}\norm{\ee_t}^2_2 \,.
\end{align*}
In the second inequality follows from the mean-value inequality and holds for any $\rho > 0$. Lemma \ref{lem:bounded-error} helps us bound the norm of $\ee_t$:
\begin{equation*}
  \expect_t\sbr*{f(\txx_{t+1})} \leq f(\txx_t) - \gamma \rbr*{1- \frac{\rho}{2}}\norm{\nabla f(\xx_t)}^2_2 + \frac{L\gamma^2\sigma^2}{2}+ \frac{\gamma^3 L^2 \sigma^2}{2 \rho}\frac{4(1-\delta)}{\delta^2} \,.
\end{equation*}
Rearranging the terms and averaging over $t$ gives for $\rho < 2$ %
\begin{align*}
  \frac{1}{T+1}\sum_{t=0}^T \norm{\nabla f(\xx_t)}^2_2 &\leq  \frac{1}{\gamma \rbr*{1- \frac{\rho}{2}}(T+1)} \sum_{t=0}^T \rbr*{\expect\sbr*{f(\txx_{t})} - \expect\sbr*{f(\txx_{t+1})}} \\&\qquad\qquad + \frac{L\gamma\sigma^2}{2 - \rho}+ \frac{ \gamma^2 L^2 \sigma^2}{\rho (2 - \rho)}\frac{4(1-\delta)}{\delta^2}\\
  &\leq \frac{f(\xx_0) - f^\star}{\gamma (T+1) (1 - \rho/2)} + \frac{L\gamma\sigma^2}{2 - \rho}+ \frac{4 \gamma^2 L^2 \sigma^2 (1- \delta)}{\rho (2 - \rho)\delta^2}\,.  \tag*{\qed}
\end{align*}
Using $\rho = 1$ gives the result as stated in Theorem \ref{thm:ecsgd-non-convex}. Note that we can choose $\rho$ arbitarily close to 0. By choosing $\rho \overset{T = \infty}{\longrightarrow} 0$, we can show that the asymptotic convergence of {\esignsgd} is in fact exactly the same as SGD.

\subsection{Proof of non-convex convergence of SGD (Remark \ref{rem:ecsgd-non-convex-rate})}
The proof is exactly as above, but with some simplifications since (essentially) $\ee_t = 0$. So using the smoothness of $f$ as before and the fact that $\expect_t[\xx_{t+1}] = \xx_t - \gamma \nabla f(\xx_t)$ we get that
\begin{align*}
  \expect_t[f(\xx_{t+1})] &\leq f(\txx_t) + \lin*{\nabla f(\xx_t)}{\expect_t\sbr*{\xx_{t+1} - \xx_t}} + \frac{L}{2}\expect_t\sbr*{\norm*{\xx_{t+1} - \xx_t}^2_2}\\
  &= f(\xx_t) - \gamma \norm{\nabla f(\xx_t)}^2 + \frac{L \sigma^2 \gamma^2}{2}\,.
\end{align*}
Now rearranging the terms and averaging over $T$ gives
Rearranging the terms and averaging over $t$ gives
\begin{align*}
  \frac{1}{T+1}\sum_{t=0}^T \norm{\nabla f(\xx_t)}^2_2 &\leq  \frac{1}{\gamma (T+1)} \sum_{t=0}^T \rbr*{\expect\sbr*{f(\txx_{t})} - \expect\sbr*{f(\txx_{t+1})}} + \frac{L\gamma\sigma^2}{2}\\
  &\leq \frac{f(\xx_0) - f^\star}{\gamma (T+1)} + \frac{L\gamma\sigma^2}{2}\,.  \tag*{\qed}
\end{align*}

\subsection{Proof of compression of unbiased estimators (Remark \ref{rem:unbiased-estimator})}
Suppose that $\expect[\cU(\vv)]=\vv$ and that $\expect\sbr*{\norm{\cU(\vv)}}^2 \leq k \norm{\vv}^2$. Then
\begin{align*}
\expect\sbr*{\norm*{\tfrac{1}{k}\cU(\vv) - \vv}^2} &= \frac{1}{k^2}\expect\sbr*{\norm{\cU(\vv)}^2} - \frac{2}{k}\lin*{\expect[\cU(\vv)]}{\vv}  + \norm{\vv}^2\\
&\leq \frac{1}{k}\norm{\vv}^2 - \frac{2}{k}\norm{\vv}^2 + \norm{\vv}^2\\
&= \rbr*{1 - \frac{1}{k}}\norm{\vv}^2\,.  \tag*{\qed}
\end{align*}

\subsection{Proof of convex convergence of {\esignsgd} (Theorem \ref{thm:ecsgd-convex})}
As in the proof of Theorem \ref{thm:ecsgd-non-convex}, we start by considering the sequence $\{\txx_t\}$ where $\txx_t = \xx_t - \ee_t$. As we saw, $\txx_{t+1} = \txx_t - \gamma \gg_t$. Suppose that $\xx_t^\star$ is an optimum solution. We will abuse notation here and use $\partial f(\xx)$ to mean any subgradient of $f$ at $\xx$.
\begin{align*}
  \expect_t\sbr*{\norm{\txx_{t+1}- \xx^\star}^2} &= \expect_t\sbr*{\norm{\txx_{t} - \gamma \gg_t - \xx^\star}^2}\\
  &= \norm{\txx_{t} - \xx^\star}^2 + \gamma^2 \expect_t\sbr*{\norm{\gg_t }^2} - 2\gamma\lin*{\expect_t\sbr*{\gg_t}}{\txx_t - \xx^\star}\\
  &\leq \norm{\txx_{t} - \xx^\star}^2 + \gamma^2 \sigma^2 - 2\gamma\lin*{\partial f(\xx_t)}{\txx_t - \xx^\star}\,.
\end{align*}
We do not want $\txx_t$ appearing in the right side of the equation and so we will replace it with $\xx_t$ and use Lemma \ref{lem:bounded-error} to bound the error:
\begin{align*}
\expect_t\sbr*{\norm{\txx_{t+1}- \xx^\star}^2} &\leq \norm{\txx_{t} - \xx^\star}^2 + \gamma^2 \sigma^2 - 2\gamma\lin*{\partial f(\xx_t)}{\xx_t - \xx^\star} + 2\gamma\lin*{\partial f(\xx_t)}{\xx_t - \txx_t}\\
&= \norm{\txx_{t} - \xx^\star}^2 + \gamma^2 \sigma^2 - 2\gamma\lin*{\partial f(\xx_t)}{\xx_t - \xx^\star} - 2\gamma\lin*{\partial f(\xx_t)}{\ee_t}\\
&\leq \norm{\txx_{t} - \xx^\star}^2 + \gamma^2 \sigma^2 - 2\gamma\lin*{\partial f(\xx_t)}{\xx_t - \xx^\star} + 2\gamma\norm{\partial f(\xx_t)}{\norm{\ee_t}}\\
&\leq \norm{\txx_{t} - \xx^\star}^2 + \gamma^2 \sigma^2 - 2\gamma\lin*{\partial f(\xx_t)}{\xx_t - \xx^\star} + \frac{4\gamma^2 \sigma\sqrt{1 - \delta}}{\delta}\norm{\partial f(\xx_t)}\,.
\end{align*}
We use Cauchy-Shwarzch in the third step, and Lemma \ref{lem:bounded-error} in the last step. We will use the loose bound $\norm{\partial f(\xx_t)} \leq \sigma$. This is the key difference between the non-smooth case and the smooth (and strongly-convex) case considered in \cite{stich2018sparsified}. In the smooth case, the term $\norm{\nabla f(\xx_t)}^2$ can be bounded by the error $f(\xx_t) - f^\star$. This implies that the last term in the equation above goes to 0 faster than $\gamma^2$, allowing the asymptotic rate to not depend on the compression quality $\delta$. However, this is not true in the non-smooth case making the dependence of the rate on $\delta$ unavoidable. We now have
\begin{equation}\label{eqn:non-smooth-ecsgd-progress}
\expect_t\sbr*{\norm{\txx_{t+1}- \xx^\star}^2} \leq \norm{\txx_{t} - \xx^\star}^2 + \gamma^2 \sigma^2 - 2\gamma\lin*{\partial f(\xx_t)}{\xx_t - \xx^\star} + \frac{4\gamma^2 \sigma^2\sqrt{1 - \delta}}{\delta}\,.
\end{equation}
Recall that $\ee_0 = 0$ and so $\txx_0 = \xx_0$. Rearranging the terms and averaging, we get
\begin{align*}
  \frac{1}{T+1} \sum_{t=0}^T \expect\sbr*{\lin*{\partial f(\xx_t)}{\xx_t - \xx^\star}} &\leq \frac{1}{2\gamma (T+1)}\sum_{t=0}^T \rbr*{\expect\sbr*{\norm{\txx_{t}- \xx^\star}^2} - \expect\sbr*{\norm{\txx_{t+1}- \xx^\star}^2}} \\
  &\qquad\qquad+ \frac{\gamma}{2} \sigma^2 + \frac{2 \gamma \sigma^2 \sqrt{1 - \delta}}{\delta}\\
  &\leq \frac{\norm{\xx_0 - \xx^\star}^2}{2\gamma (T+1)} + \gamma \sigma^2\rbr*{\frac{1}{2} + \frac{2\sqrt{1 - \delta}}{\delta}}\,.
\end{align*}
To finish the proof, we have to simply use the convexity of $f$ twice on the left hand side of the above inequality:
\[
  \frac{1}{T+1} \sum_{t=0}^T \expect\sbr*{\lin*{\partial f(\xx_t)}{\xx_t - \xx^\star}} \geq \frac{1}{T+1} \sum_{t=0}^T f(\xx_t ) - f(\xx^\star) \geq  f( \frac{1}{T+1} \sum_{t=0}^T \xx_t) - f(\xx^\star) = f(\bar \xx_T) - f(\xx^\star)\,.
\]

For the standard rate of SGD in remark \ref{rem:ecgsd-non-smooth-rate}, we just set $\delta = 0$. \hfill $\qed$

\subsection{Proof relating linear span of gradients to pseudo-inverse (Lemma \ref{lem:span-pseudo-inverse})}
Recall that $\mA \in \real^{n\times d}$ for $n < d$. Assume without loss of generality that the rows of $\mA$ are linearly independent and hence $\mA$ is of rank $n$. The stochastic gradient for $f(\xx) = \norm{\mA\xx - \bb}^2$ is of the form $\sum \alpha_i \mA_{i,:}$ where $\mA_{i,:}$ indicates the $i$th column of $\mA$. If $\xx_t$ is in the linear span of the stochastic gradients, then there exists a vector $\balpha_t \in \R^n$ such that
\[
  \xx_t = \mA^\top\balpha_t\,.
\]
Suppose $\xx^\star$ is the solution reached. Then $\mA \xx^\star = \bb$ and also $\xx^\star = \mA^\top\balpha^\star$ for some $\balpha^\star$. Hence $\balpha^\star$ must satisfy
\[
  \mA \mA^\top \balpha^\star = \bb\,.
\]
Since the rank of $\mA$ is $n$, the matrix $\mA \mA^\top \in \real^{n\times n}$ is full-rank and invertible. This means that there exists an unique solution to $\balpha^\star$ and $\xx^\star$:
\begin{align*}
  \balpha^\star = \rbr*{\mA \mA^\top}^{-1}\bb \quad \text{and} \quad \xx^\star = \mA^\top\balpha^\star = \mA^\top \rbr*{\mA \mA^\top}^{-1}\bb\,. \tag*{\qed}
\end{align*}

%% file: fixed-signSGD19.bbl
\begin{thebibliography}{47}
\providecommand{\natexlab}[1]{#1}
\providecommand{\url}[1]{\texttt{#1}}
\expandafter\ifx\csname urlstyle\endcsname\relax
  \providecommand{\doi}[1]{doi: #1}\else
  \providecommand{\doi}{doi: \begingroup \urlstyle{rm}\Url}\fi

\bibitem[Alistarh et~al.(2017)Alistarh, Grubic, Li, Tomioka, and
  Vojnovic]{alistarh2017quantized}
Dan Alistarh, Demjan Grubic, Jerry Li, Ryota Tomioka, and Milan Vojnovic.
\newblock Qsgd: Communication-efficient sgd via gradient quantization and
  encoding.
\newblock In \emph{Advances in Neural Information Processing Systems (NIPS)},
  2017.

\bibitem[Arpit et~al.(2017)Arpit, Jastrz{\k{e}}bski, Ballas, Krueger, Bengio,
  Kanwal, Maharaj, Fischer, Courville, Bengio, et~al.]{arpit2017closer}
Devansh Arpit, Stanis{\l}aw Jastrz{\k{e}}bski, Nicolas Ballas, David Krueger,
  Emmanuel Bengio, Maxinder~S Kanwal, Tegan Maharaj, Asja Fischer, Aaron
  Courville, Yoshua Bengio, et~al.
\newblock A closer look at memorization in deep networks.
\newblock In \emph{International Conference on Machine Learning (ICML)}, 2017.

\bibitem[Balles and Hennig(2018)]{balles2017dissecting}
Lukas Balles and Philipp Hennig.
\newblock Dissecting adam: The sign, magnitude and variance of stochastic
  gradients.
\newblock In \emph{Internation Conference on Machine Learning (ICML)}, 2018.

\bibitem[Bernstein et~al.(2018)Bernstein, Wang, Azizzadenesheli, and
  Anandkumar]{bernstein2018signsgd}
Jeremy Bernstein, Yu-Xiang Wang, Kamyar Azizzadenesheli, and Anima Anandkumar.
\newblock signsgd: compressed optimisation for non-convex problems.
\newblock In \emph{Internation Conference on Machine Learning (ICML)}, 2018.

\bibitem[Bernstein et~al.(2019)Bernstein, Zhao, Azizzadenesheli, and
  Anandkumar]{bernstein2019iclr}
Jeremy Bernstein, Jiawei Zhao, Kamyar Azizzadenesheli, and Anima Anandkumar.
\newblock sign{SGD} with majority vote is communication efficient and fault
  tolerant.
\newblock In \emph{International Conference on Learning Representations
  (ICLR)}, 2019.

\bibitem[Bottou(2010)]{Bottou2010:sgd}
L{\'e}on Bottou.
\newblock Large-scale machine learning with stochastic gradient descent.
\newblock In Yves Lechevallier and Gilbert Saporta, editors, \emph{Proceedings
  of COMPSTAT'2010}, pages 177--186, Heidelberg, 2010. Physica-Verlag HD.
\newblock ISBN 978-3-7908-2604-3.

\bibitem[Carlson et~al.(2015)Carlson, Cevher, and Carin]{Carlson:2015to}
David Carlson, Volkan Cevher, and Lawrence Carin.
\newblock {Stochastic Spectral Descent for Restricted Boltzmann Machines}.
\newblock In \emph{International Conference on Artificial Intelligence and
  Statistics (AISTATS)}, pages 111--119, February 2015.

\bibitem[Chen and Gu(2019)]{chen2019padam}
Jinghui Chen and Quanquan Gu.
\newblock Padam: Closing the generalization gap of adaptive gradient methods in
  training deep neural networks.
\newblock In \emph{International Conference on Learning Representations
  (ICLR)}, 2019.

\bibitem[Chilimbi et~al.(2014)Chilimbi, Suzue, Apacible, and
  Kalyanaraman]{chilimbi2014project}
Trishul~M Chilimbi, Yutaka Suzue, Johnson Apacible, and Karthik Kalyanaraman.
\newblock Project adam: Building an efficient and scalable deep learning
  training system.
\newblock In \emph{OSDI}, volume~14, pages 571--582, 2014.

\bibitem[Cortes and Vapnik(1995)]{cortes1995support}
Corinna Cortes and Vladimir Vapnik.
\newblock Support-vector networks.
\newblock \emph{Machine learning}, 20\penalty0 (3):\penalty0 273--297, 1995.

\bibitem[Dean et~al.(2012)Dean, Corrado, Monga, Chen, Devin, Mao, Senior,
  Tucker, Yang, Le, et~al.]{dean2012large}
Jeffrey Dean, Greg Corrado, Rajat Monga, Kai Chen, Matthieu Devin, Mark Mao,
  Andrew Senior, Paul Tucker, Ke~Yang, Quoc~V Le, et~al.
\newblock Large scale distributed deep networks.
\newblock In \emph{Advances in Neural Information Processing Systems (NIPS)},
  pages 1223--1231, 2012.

\bibitem[Dinh et~al.(2017)Dinh, Pascanu, Bengio, and Bengio]{dinh2017sharp}
Laurent Dinh, Razvan Pascanu, Samy Bengio, and Yoshua Bengio.
\newblock Sharp minima can generalize for deep nets.
\newblock \emph{arXiv preprint arXiv:1703.04933}, 2017.

\bibitem[Duchi et~al.(2011)Duchi, Hazan, and Singer]{duchi2011adaptive}
John Duchi, Elad Hazan, and Yoram Singer.
\newblock Adaptive subgradient methods for online learning and stochastic
  optimization.
\newblock \emph{Journal of Machine Learning Research}, 12\penalty0
  (Jul):\penalty0 2121--2159, 2011.

\bibitem[Goyal et~al.(2017)Goyal, Dollar, Girshick, Noordhuis, Wesolowski,
  Kyrola, Tulloch, Jia, and He]{goyal2017accurate}
Priya Goyal, Piotr Dollar, Ross Girshick, Pieter Noordhuis, Lukasz Wesolowski,
  Aapo Kyrola, Andrew Tulloch, Yangqing Jia, and Kaiming He.
\newblock Accurate, large minibatch sgd: training imagenet in 1 hour.
\newblock \emph{arXiv preprint arXiv:1706.02677}, 2017.

\bibitem[Gugger and Howard(2018)]{Gugger2018AdamW}
Sylvain Gugger and Jeremy Howard.
\newblock Adamw and super-convergence is now the fastest way to train neural
  nets.
\newblock \url{https://www.fast.ai/2018/07/02/adam-weight-decay/}, 2018.
\newblock Accessed: 2019-01-17.

\bibitem[Gunasekar et~al.(2018)Gunasekar, Lee, Soudry, and
  Srebro]{gunasekar2018characterizing}
Suriya Gunasekar, Jason Lee, Daniel Soudry, and Nathan Srebro.
\newblock Characterizing implicit bias in terms of optimization geometry.
\newblock In \emph{International Conference on Machine Learning (ICML)}, 2018.

\bibitem[He et~al.(2016{\natexlab{a}})He, Zhang, Ren, and Sun]{he2016deep}
Kaiming He, Xiangyu Zhang, Shaoqing Ren, and Jian Sun.
\newblock Deep residual learning for image recognition.
\newblock In \emph{Proceedings of the IEEE conference on computer vision and
  pattern recognition (CVPR)}, pages 770--778, 2016{\natexlab{a}}.

\bibitem[He et~al.(2016{\natexlab{b}})He, Zhang, Ren, and Sun]{he2016identity}
Kaiming He, Xiangyu Zhang, Shaoqing Ren, and Jian Sun.
\newblock Identity mappings in deep residual networks.
\newblock In \emph{European Conference on Computer Vision (ECCV)}, pages
  630--645. Springer, 2016{\natexlab{b}}.

\bibitem[Im et~al.(2016)Im, Tao, and Branson]{im2016empirical}
Daniel~Jiwoong Im, Michael Tao, and Kristin Branson.
\newblock An empirical analysis of the optimization of deep network loss
  surfaces.
\newblock \emph{arXiv preprint arXiv:1612.04010}, 2016.

\bibitem[Kawaguchi et~al.(2017)Kawaguchi, Kaelbling, and
  Bengio]{kawaguchi2017generalization}
Kenji Kawaguchi, Leslie~Pack Kaelbling, and Yoshua Bengio.
\newblock Generalization in deep learning.
\newblock \emph{arXiv preprint arXiv:1710.05468}, 2017.

\bibitem[Kingma and Ba(2015)]{kingma2014adam}
Diederik~P Kingma and Jimmy Ba.
\newblock Adam: A method for stochastic optimization.
\newblock In \emph{International Conference on Learning Representations
  (ICLR)}, 2015.

\bibitem[Krizhevsky and Hinton(2009)]{krizhevsky2009learning}
Alex Krizhevsky and Geoffrey Hinton.
\newblock Learning multiple layers of features from tiny images.
\newblock Technical report, Technical Report, University of Toronto, Toronto.,
  2009.

\bibitem[Krizhevsky et~al.(2012)Krizhevsky, Sutskever, and
  Hinton]{krizhevsky2012imagenet}
Alex Krizhevsky, Ilya Sutskever, and Geoffrey~E Hinton.
\newblock Imagenet classification with deep convolutional neural networks.
\newblock In \emph{Advances in Neural Information Processing Systems (NIPS)},
  pages 1097--1105, 2012.

\bibitem[LeCun et~al.(2015)LeCun, Bengio, and Hinton]{lecun2015deep}
Yann LeCun, Yoshua Bengio, and Geoffrey Hinton.
\newblock Deep learning.
\newblock \emph{nature}, 521\penalty0 (7553):\penalty0 436, 2015.

\bibitem[Li et~al.(2018)Li, Xu, Taylor, Studer, and
  Goldstein]{li2018visualizing}
Hao Li, Zheng Xu, Gavin Taylor, Christoph Studer, and Tom Goldstein.
\newblock Visualizing the loss landscape of neural nets.
\newblock In \emph{Advances in Neural Information Processing Systems
  (NeurIPS)}, 2018.

\bibitem[Li et~al.(2014)Li, Andersen, Park, Smola, Ahmed, Josifovski, Long,
  Shekita, and Su]{li2014scaling}
Mu~Li, David~G Andersen, Jun~Woo Park, Alexander~J Smola, Amr Ahmed, Vanja
  Josifovski, James Long, Eugene~J Shekita, and Bor-Yiing Su.
\newblock Scaling distributed machine learning with the parameter server.
\newblock In \emph{OSDI}, volume~14, pages 583--598, 2014.

\bibitem[Lin et~al.(2018)Lin, Han, Mao, Wang, and Dally]{lin2017deep}
Yujun Lin, Song Han, Huizi Mao, Yu~Wang, and William~J Dally.
\newblock Deep gradient compression: Reducing the communication bandwidth for
  distributed training.
\newblock In \emph{International Conference on Learning Representations
  (ICLR)}, 2018.

\bibitem[Liu et~al.(2018)Liu, Chen, Chen, and Hong]{liu2018signsgd}
Sijia Liu, Pin-Yu Chen, Xiangyi Chen, and Mingyi Hong.
\newblock signsgd via zeroth-order oracle.
\newblock In \emph{International Conference on Learning Representations
  (ICLR)}, 2018.

\bibitem[Luo et~al.(2019)Luo, Xiong, and Liu]{luo2019adaptive}
Liangchen Luo, Yuanhao Xiong, and Yan Liu.
\newblock Adaptive gradient methods with dynamic bound of learning rate.
\newblock In \emph{International Conference on Learning Representations
  (ICLR)}, 2019.

\bibitem[Paszke et~al.(2017)Paszke, Gross, Chintala, and
  Chanan]{paszke2017pytorch}
Adam Paszke, Sam Gross, Soumith Chintala, and Gregory Chanan.
\newblock Pytorch, 2017.

\bibitem[Reddi et~al.(2018)Reddi, Kale, and Kumar]{reddi2018convergence}
Sashank~J Reddi, Satyen Kale, and Sanjiv Kumar.
\newblock On the convergence of adam and beyond.
\newblock In \emph{International Conference on Learning Representations
  (ICLR)}, 2018.

\bibitem[Riedmiller and Braun(1993)]{riedmiller1993direct}
Martin Riedmiller and Heinrich Braun.
\newblock A direct adaptive method for faster backpropagation learning: The
  rprop algorithm.
\newblock In \emph{Neural Networks, 1993., IEEE International Conference on},
  pages 586--591. IEEE, 1993.

\bibitem[Robbins and Monro(1951)]{robbins1951stochastic}
Herbert Robbins and Sutton Monro.
\newblock {A Stochastic Approximation Method}.
\newblock \emph{The Annals of Mathematical Statistics}, 22\penalty0
  (3):\penalty0 400--407, September 1951.

\bibitem[Schmidhuber(2015)]{schmidhuber2015deep}
J{\"u}rgen Schmidhuber.
\newblock Deep learning in neural networks: An overview.
\newblock \emph{Neural networks}, 61:\penalty0 85--117, 2015.

\bibitem[Seide et~al.(2014)Seide, Fu, Droppo, Li, and Yu]{seide20141}
Frank Seide, Hao Fu, Jasha Droppo, Gang Li, and Dong Yu.
\newblock 1-bit stochastic gradient descent and its application to
  data-parallel distributed training of speech dnns.
\newblock In \emph{Fifteenth Annual Conference of the International Speech
  Communication Association}, 2014.

\bibitem[Simonyan and Zisserman(2014)]{simonyan2014very}
Karen Simonyan and Andrew Zisserman.
\newblock Very deep convolutional networks for large-scale image recognition.
\newblock \emph{arXiv preprint arXiv:1409.1556}, 2014.

\bibitem[Soudry et~al.(2018)Soudry, Hoffer, Nacson, Gunasekar, and
  Srebro]{soudry2018implicit}
Daniel Soudry, Elad Hoffer, Mor~Shpigel Nacson, Suriya Gunasekar, and Nathan
  Srebro.
\newblock The implicit bias of gradient descent on separable data.
\newblock \emph{Journal of Machine Learning Research}, 19\penalty0 (70), 2018.

\bibitem[Stich et~al.(2018)Stich, Cordonnier, and Jaggi]{stich2018sparsified}
Sebastian~U Stich, Jean-Baptiste Cordonnier, and Martin Jaggi.
\newblock Sparsified sgd with memory.
\newblock In \emph{Advances in Neural Information Processing Systems
  (NeurIPS)}, 2018.

\bibitem[Strom(2015)]{strom2015scalable}
Nikko Strom.
\newblock Scalable distributed dnn training using commodity gpu cloud
  computing.
\newblock In \emph{Sixteenth Annual Conference of the International Speech
  Communication Association}, 2015.

\bibitem[Valiant(1984)]{valiant1984theory}
Leslie~G Valiant.
\newblock A theory of the learnable.
\newblock \emph{Communications of the ACM}, 27\penalty0 (11):\penalty0
  1134--1142, 1984.

\bibitem[Wang et~al.(2018)Wang, Sievert, Liu, Charles, Papailiopoulos, and
  Wright]{wang2018atomo}
Hongyi Wang, Scott Sievert, Shengchao Liu, Zachary Charles, Dimitris
  Papailiopoulos, and Stephen Wright.
\newblock Atomo: Communication-efficient learning via atomic sparsification.
\newblock In \emph{Advances in Neural Information Processing Systems
  (NeurIPS)}, 2018.

\bibitem[Wen et~al.(2017)Wen, Xu, Yan, Wu, Wang, Chen, and Li]{wen2017terngrad}
Wei Wen, Cong Xu, Feng Yan, Chunpeng Wu, Yandan Wang, Yiran Chen, and Hai Li.
\newblock Terngrad: Ternary gradients to reduce communication in distributed
  deep learning.
\newblock In \emph{Advances in Neural Information Processing Systems (NIPS)},
  pages 1509--1519, 2017.

\bibitem[Wilson et~al.(2017)Wilson, Roelofs, Stern, Srebro, and
  Recht]{wilson2017marginal}
Ashia~C Wilson, Rebecca Roelofs, Mitchell Stern, Nati Srebro, and Benjamin
  Recht.
\newblock The marginal value of adaptive gradient methods in machine learning.
\newblock In \emph{Advances in Neural Information Processing Systems (NIPS)},
  pages 4148--4158, 2017.

\bibitem[Wu et~al.(2018)Wu, Huang, Huang, and Zhang]{wu2018error}
Jiaxiang Wu, Weidong Huang, Junzhou Huang, and Tong Zhang.
\newblock Error compensated quantized sgd and its applications to large-scale
  distributed optimization.
\newblock In \emph{International Conference on Machine Learning (ICML)}, pages
  5321--5329, 2018.

\bibitem[Zaheer et~al.(2018)Zaheer, Reddi, Sachan, Kale, and
  Kumar]{zaheer2018adaptive}
Manzil Zaheer, Sashank Reddi, Devendra Sachan, Satyen Kale, and Sanjiv Kumar.
\newblock Adaptive methods for nonconvex optimization.
\newblock In \emph{Advances in Neural Information Processing Systems
  (NeurIPS)}, pages 9815--9825, 2018.

\bibitem[Zhang et~al.(2017)Zhang, Bengio, Hardt, Recht, and
  Vinyals]{zhang2016understanding}
Chiyuan Zhang, Samy Bengio, Moritz Hardt, Benjamin Recht, and Oriol Vinyals.
\newblock Understanding deep learning requires rethinking generalization.
\newblock In \emph{International Conference on Learning Representations
  (ICLR)}, 2017.

\bibitem[Zhang et~al.(2018)Zhang, Vinyals, Munos, and Bengio]{zhang2018study}
Chiyuan Zhang, Oriol Vinyals, Remi Munos, and Samy Bengio.
\newblock A study on overfitting in deep reinforcement learning.
\newblock \emph{arXiv preprint arXiv:1804.06893}, 2018.

\end{thebibliography}
